\documentclass[10pt]{article}

\usepackage{fullpage}
\usepackage{hyperref}
\usepackage[numbers]{natbib}
\usepackage{tabularx}
\usepackage{booktabs}
\usepackage[dvipsnames]{xcolor}
\usepackage{amssymb,amsthm,amsfonts,amsmath}
\usepackage{mathtools}
\usepackage{nicefrac}
\usepackage{algorithm,algorithmic}
\usepackage{graphicx}

\newtheorem{definition}{Definition}

\newtheorem{theorem}{Theorem}
\newtheorem{lemma}{Lemma}
\newtheorem{proposition}{Proposition}

\newtheorem{remark}{Remark}

\newcommand{\eqdef}{\coloneqq} 
\newcommand{\cO}{{\cal O}}
\newcommand{\cC}{{\cal C}}

\newcommand{\R}{\mathbb{R}}
\newcommand{\N}{\mathbb{N}}
\newcommand{\C}{\mathbb{C}}

\newcommand{\E}{\mathbb{E}}
\newcommand{\B}{\mathbb{B}}
\newcommand{\U}{\mathbb{U}}
\newcommand{\sphere}{\mathbb{S}}

\def\Pr{{\rm Prob}}

\DeclareMathOperator*{\sign}{{\rm sign}}

\def\<{\left\langle}
\def\>{\right\rangle}
\def\[{\left[}
\def\]{\right]}
\def\({\left(}
\def\){\right)}

\hypersetup{
	colorlinks=true,       
	linkcolor=blue,        
	citecolor=blue,        
	filecolor=magenta,     
	urlcolor=blue
}

\begin{document}
	
\title{\bf Optimal Gradient Compression \\for Distributed and Federated Learning}
\author{\qquad Alyazeed Albasyoni   \qquad Mher Safaryan \qquad Laurent Condat \qquad  Peter Richt{\'a}rik \\
\phantom{XXX} \\
{\em King Abdullah University of Science and Technology (KAUST)}
}

\date{September 28, 2020}
\maketitle

\begin{abstract}
Communicating information, like gradient vectors, between computing nodes in distributed and federated learning is typically an unavoidable burden, resulting in scalability issues. Indeed, communication might be slow and costly. Recent advances in communication-efficient training algorithms have reduced this bottleneck by using compression techniques, in the form of sparsification, quantization, or low-rank approximation. Since compression is a lossy, or inexact, process, the iteration complexity is typically worsened; but the total communication complexity can improve significantly, possibly leading to large computation time savings. In this paper, we investigate the fundamental trade-off between the number of bits needed to encode compressed vectors and the compression error. 
We perform both worst-case and average-case analysis, providing tight lower bounds. In the worst-case analysis, we introduce an efficient compression operator, {\em Sparse Dithering}, which is very close to the lower bound. In the average-case analysis, we design a simple compression operator,  {\em Spherical Compression}, which naturally achieves the lower bound. Thus, our new compression schemes significantly outperform the state of the art. We conduct numerical experiments to illustrate this improvement.
\end{abstract}

\tableofcontents

\section{Introduction}

Due to the necessity of huge amounts of data to achieve high-quality machine learning models \cite{Sch,Vaswani2019-overparam}, modern large-scale training procedures are executed in a distributed environment \cite{bekkerman2011scaling,vogels}. In such a setup, both storage and computation needs are reduced, as the overall data (potentially too big to fit into a single machine) is partitioned among the nodes and computation is carried out in parallel. However, in order to keep the consensus across the network, compute nodes have to exchange some information about their local progress \cite{Stich2018:localsgd,localSGD-AISTATS2020,basu}.
The demand of information communication between all machines in a distributed setup is typically a burden, resulting in a scalability issue commonly referred to as {\em communication bottleneck} \cite{SFDLY, ZLKALZ, LHMWD}. To reduce the amount of information to be transferred, information is passed in a compressed or inexact form. {\em Information lossy compression} is a common practice, where original information is encoded approximately with essentially fewer bits, while introducing additional controllable distortion into the decoded message.

In the context of {\em Federated Learning} \cite{FEDLEARN,FL2017-AISTATS,Karimireddy2019}, communication between devices arises naturally, as data is initially decentralized and should remain so, for privacy purposes. Actually,
it might be desirable for each unit, or client, to compress/encode/encrypt 
the information they are going to share, in order to minimize private data disclosures.
Another practical scenario where compression methods are useful is when storage capabilities are scarce or there is no need to save complete versions of the data. In such cases, the representation of the data (encoding and decoding schemes) can be optimized and with little or no precision loss, one can allocate significantly less memory space.

\subsection{Related Work}

Recently, substantial amount of work has been devoted to the advances of communication-efficient training algorithms by utilizing various types of compression mechanisms, such as sparsification \cite{tonko,WSLCPW,Dryden2016:topk}, quantization \cite{AGLTV,terngrad,Cnat} and low-rank approximation \cite{vogels}. Typically, the information communicated by computing nodes consists of local gradients, to which compression operators are applied. For example, one popular example of such compression operator is Top-$k$ \cite{Alistarh-SparsGradMethods2018}, which transfers only $k$ coordinates of the gradient with largest magnitudes. 

The theoretical foundation of lossy compression has long history and is based on {\em Rate-Distortion Theory} introduced by Shannon in his seminal papers \cite{shannon-1,shannon-2}. Recently, rate-distortion theory has been utilized in the context of model compression \cite{GLWO-modelcompression,YWSV-modelcompression}. In contrast to this, another line of research is devoted to the lossless compression methods which is rooted in {\em Shannon's source coding theorem} \cite{elements_IT}. Both approaches exploit statistical properties of the input messages for analyses, which differs from our setting.

We investigate the problem of lossy compression, namely encoding vectors $x\in\R^d$ without prior knowledge on the distribution, for any $d\geq 1$,  into as few bits as possible, while introducing as little distortion as possible. Formally, we measure the distortion of a (possibly randomized) compression operator $\mathcal{C}\colon\R^d\to\R^d$ by its constant $\alpha\in[0,1]$  such that $\E\[\|\mathcal{C}(x)-x\|^2\] \le \alpha\|x\|^2$ for every $x$, where the norm is the Euclidean norm (see Definitions \ref{def:omega-compressor}, \ref{def:biased} and \ref{def:strict-contraction} for details).  We denote by $b$ the number of bits  (in the worst case or in expectation) needed to encode $\mathcal{C}(x)$.
Intuitively, $b$  and $\alpha$  cannot be too small at the same time: they are antagonistic and ruled by a fundamental rate-distortion trade-off. As a matter of fact, as shown in \cite{up_kashin_2020},
the following lower bound, referred to as {\em uncertainty principle for communication compression}, holds (if omitted, the base of $\log$ is assumed to be $2$):
\begin{equation}\label{up-alpha}
\alpha\, 4^{\nicefrac{b}{d}} \ge 1 \quad\text{or, equivalently,}\quad b\ge d \, \frac{1}{2}\log\frac{1}{\alpha}.
\end{equation}
In this work, we investigate this trade-off more deeply.
We perform two types of analyses: {\em worst case analysis (WCA)} and {\em average case analysis (ACA)}.
Then, capitalizing on this new knowledge, we design new efficient compression schemes.
Note that our derivations  deal with real numbers, compressed using a finite number of bits. We should keep in mind that numbers are represented by finite-precision, say 32 bits, floats in computers. We can safely omit this aspect in the derivations; we discuss this point in more details in the Appendix.

\subsection{Contributions}

Here we summarize our key contributions.

$\bullet$ {\bf (WCA) Tighter bounds on minimal communication.}
First, we construct a compression scheme with $\alpha$ distortion and $b$ encoding bits (in the worst case), which satisfies
\begin{equation}\label{covering-example-bound}
\alpha \, 4^{\nicefrac{b}{d}} \le \text{\rm poly}(d)^{\nicefrac{1}{d}} \quad\text{or}\quad b\le d \, \frac{1}{2}\log\frac{1}{\alpha} + \cO(\log{d}).
\end{equation}
This implies the asymptotic tightness of the bound (\ref{up-alpha}) as the dimension $d$ grows (see Theorem \ref{thm:tight-construction-alpha}). Then, we investigate the minimal number of bits (in the worst case) $b^*(\alpha, d)$ as a function of distortion $\alpha$ and dimension $d$ proving that
\begin{equation*}
b^*(\alpha, d) = -\log P(\alpha,d) + \log d + \frac{1}{2}\log\log d + e,
\end{equation*}
where $P(\alpha, d) \eqdef \frac{1}{2}I_{\alpha}(\frac{d-1}{2}, \frac{1}{2})$ with $I_{\alpha}$ being the regularized incomplete beta function, and $e$ is negligible additive error with $|e| \le \frac{1}{2}\log\log d + \cO(1)$ (see Theorem \ref{thm:tighter-bounds}), as opposed to $\cO(\log{d})$ in (\ref{up-alpha}) and (\ref{covering-example-bound}).

$\bullet$ {\bf (WCA) Near optimal and practical compressor.}
Motivated by these lower bounds we turn to the construction of a compression method which would be optimal and implementable in high dimensions. The example compression schemes in Theorem \ref{thm:tight-construction-alpha} ensuring (\ref{covering-example-bound}) or in Theorem \ref{thm:tighter-bounds} are optimal but impractical, due to the exponential computation time to compress a vector. To make the scheme efficient, we slightly depart from the optimal boundary and propose a new efficient compression method---{\em Sparse Dithering (SD)}.
Both deterministic (biased) and randomized (unbiased) versions of SD are analyzed, and comparisons with existing methods are made, showing that we outperform the state of the art.  
In the special case, the encoding of deterministic SD with $\alpha=\nicefrac{1}{10}$ distortion requires at most $30+ \log d + 3.35d$ bits, which is optimal within $1.69d$ additional bits (see Theorem \ref{thm:DSD}).

$\bullet$ {\bf (ACA) Lower bound on average communication.}
Switching to the average case analysis, we establish a lower bound $-\log P(\alpha,d)\le B$ on the expected number of bits $B$ needed to encode a compression operator from $\C(\alpha)$ (see Definition \ref{def:strict-contraction}).

$\bullet$ {\bf (ACA) Compressor with optimal average communication.}
As an attempt to reach the lower bound obtained in the average case analysis, we first analyze the randomized (and unbiased) version of SD. We prove that with variance $\omega>0$ it requires at most
$$
30 + \log d + \(\log 3 + \frac{1}{2\sqrt{\omega}}\)d
$$
bits in expectation (see Theorem \ref{thm:RSD}). In the special case of $\omega=\nicefrac{1}{4}$, it provides $\approx 9.9\times$ bandwidth savings. However, this scheme is suboptimal with respect to the lower bound. We finally present a simple compression operator--{\em Spherical Compression}--which attains the lower bound with less than $3$ extra bits, namely it communicates $B<-\log P(\alpha,d)+3$ bits in expectation (see Theorem \ref{thm:SC}).

\section{Classes of Compression Operators}\label{sec:operator-classes}

Here we formally define and perform preliminary analysis for three general classes of compression operators, that will be considered throughout the paper. We start with the most common and well studied class of unbiased compressors \cite{qsgd,terngrad,khirirat2018distributed,tonko}.

\begin{definition}[$\omega$-compressors]\label{def:omega-compressor}
We denote by $\U(\omega)$ the class of unbiased compression operators $\mathcal{C}\colon\R^d\to\R^d$ with variance $\omega\ge 0$; that is, $\E\[\mathcal{C}(x)\] = x$ and
\begin{equation}\label{class-unbiased}
\E\[\|\mathcal{C}(x)-x\|^2\] \le \omega\|x\|^2, \ \ \forall x\in\R^d.
\end{equation}
\end{definition}

Another  broad class of compressions operators, for which compressed learning algorithms have been successfully analysed \cite{karimireddy2019error,stich2019,zheng,bez20}, is the class of biased operators, which are contractive in expectation.

\begin{definition}[$\alpha$-contractive operators]\label{def:biased}
We denote by $\B(\alpha)$ the class of (possibly biased and randomized) compression operators $\mathcal{C}\colon\R^d\to\R^d$ with $\alpha\in[0,1]$-contractive property; that is,
\begin{equation}\label{class-biased}
\E\[\|\mathcal{C}(x)-x\|^2\] \le \alpha\|x\|^2, \qquad\forall x\in\R^d.
\end{equation}
\end{definition}

Analogous to parameter $\omega$ for the variance, the parameter $\alpha$ is referred to as normalized variance or distortion threshold\footnote{note that the definition of distortion in rate--distortion theory is slightly different than what we define.}. It has been shown, that the class $\U(\omega)$ can be embedded into $\B(\alpha)$. Specifically, if $\cC\in\U(\omega)$ then $\frac{1}{\omega+1}\cC\in\B(\frac{\omega}{\omega+1})$ (see e.g. Lemma 1 in \cite{up_kashin_2020}).
We will also consider the subclass of strictly contractive operators which, compared to operators from $\B(\alpha)$, are contractive for all realizations rather than in expectation: 

\begin{definition}[Strictly $\alpha$-contractive operators]\label{def:strict-contraction}
We denote by $\C(\alpha)$ the class of (possibly biased and randomized) compression operators $\mathcal{C}\colon\R^d\to\R^d$ with $\alpha\in[0,1]$-strictly contractive property; that is,
\begin{equation}\label{class-strict-contraction}
\|\mathcal{C}(x)-x\|^2 \le \alpha\|x\|^2, \qquad \forall x\in\R^d.
\end{equation}
\end{definition}

\subsection{Compression operator as composition of encoder and decoder}

Generally speaking, compression is a two-sided notion, in the sense that one end encodes the message, while the other end decodes it to estimate the original information. An encoder is any mapping $E\colon\R^d\to\{0,1\}^*$ which maps a given vector $x\in\R^d$ to some finite word from the set of all finite words $\{0,1\}^*$ with the binary alphabet $\{0,1\}$. A decoder, on the other hand, is a mapping $D\colon\{0,1\}^*\to\R^d$ which aims to reconstruct the initial vector $x\in\R^d$ from the finite binary codeword $E(x)$. Thus, a compression operator $\cC\colon\R^d\to\R^d$ can be decomposed into an encoder and decoder so that $\cC(x)=D(E(x))$. The number of bits needed to transfer a compressed version of $x\in\R^d$ is the length $|E(x)|$ of the binary word $E(x)$. In the worst case analysis we are interested in the length of the longest codeword $\sup_{x,\cC(x)}|E(x)|$, while in average case analysis we investigate the size of the longest expected codeword $\sup_{x}\E_{\cC}\[|E(x)|\]$.

Notice that a compression operator from any of the three classes requires countably many bits in order to encode points near $x=0$ and $x=\infty$. We address this issue in the Appendix by considering relaxed classes of compression operators capturing finite representation of a single float in machines. From now on, we exclude trivial cases $\omega=0,\;\alpha\in\{0,1\}$ and assume $\omega>0,\;\alpha\in(0,1)$.

\subsection{Two senses of optimality for compression}

It is worth distinguishing between optimality within a class in a single step of communication and optimality of total communication throughout the optimization process leading to $\epsilon$-accuracy, e.g.  $\frac{\|x^t-x^\star\|^2}{\|x^0-x^\star\|^2}\leq \epsilon$ for a prescribed $\epsilon$, where $t$ is the iteration counter. Our theoretical contributions mainly deal with the first sense of optimality. Regarding the second view of optimality, the following proposition shows that Compressed Gradient Descent (CGD) can converge at significantly different speeds for different operators from $\B(\alpha)$.


\begin{proposition}\label{prop:cgd}
If $\mathcal C \in \B(\alpha)$, the iteration complexity of CGD is $\frac{1}{1-\alpha}$ times bigger than for GD; that is CGD needs $\frac{1}{1-\alpha}$ times more iterations than GD to obtain the same $\epsilon$-accuracy. Moreover, if $\mathcal C$ is additionally unbiased, then only $1+\alpha$ times more iterations are sufficient.
\end{proposition}

Thus, if we aim to minimize the total communication complexity ensuring convergence to $\epsilon$-accuracy, then the optimal operator $\cC^*$ should be either unbiased, or it will need to satisfy not only the direct condition, $\E[|E_{\cC^*}(x)|] \leq \E[|E_{\cC}(x)|]$ for all operators $\cC \in \B(\alpha)$, but also the additional condition $\E[|E_{\cC^*}(x)|] \leq (1-\alpha^2)\E[|E_{\mathcal{U}}(x)|]$ for all unbiased operators $\mathcal{U}\in\B(\alpha)$. It is important to see that when $\alpha$ is close to $1$, then this additional constraint is hard to satisfy when $\mathcal C^*$ is not unbiased. For $\alpha < 1$, we show that this is indeed the case by obtaining an optimal biased operator $\mathcal C^* \in \B(\alpha)$, which we call \emph{Spherical Compression}, and another unbiased one, which we call \emph{Sparse Dithering}. We show that the latter is more suitable in practice due to its unbiasedness, and hence, convergence occurs in much fewer iterations, and that this is most pronounced when $\alpha$ is close to 1. In addition to being computationally efficient, we show that Sparse Dithering can guarantee reducing the total training communication by $\approx 9.9\times$ compared to full precision gradient communication of $32$-bits floats.

\subsection{Dimension-tolerant compression schemes.} By dimension-tolerant compression, we mean a collection of operators $\bar{\cC} = \(\cC_d\)_{d\ge d_0}$ that can be used to compress vectors $x\in\R^d$ for any $d\ge d_0$ and there exists a non-trivial fixed upper bound ($\bar{\omega}<\infty$ or $\bar{\alpha}<1$) for variances ($\omega_d$ or $\alpha_d$), i.e. $\omega_d\le\bar{\omega}<\infty$ or $\alpha_d\le\bar{\alpha}<1$ for any $d\ge d_0$.

Below we show that for such collection of compression schemes, it is necessary and sufficient to use at least a constant amount of bits per dimension on average and this constant can be arbitrarily small.

\begin{theorem}\label{thm:dim-tol-op} The following holds:

\begin{itemize}
\item[(i)] If $\bar{\cC} = \(\cC_d\)_{d\ge d_0}$ is a dimension-tolerant compression composed of operators from $\U(\omega)$ ($\B(\alpha)$ or $\C(\alpha)$), then there exists a positive constant $c>0$ (independent of $d$) such that for any $d\ge d_0$ at least $cd$ bits are required in the worst case to encode $\cC_d(x)\in\R^d$ for any $x\in\R^d$.

\item[(ii)] Let $c>0$ be a fixed positive constant. Then there exists a dimension-tolerant compression $\bar{\cC} = \(\cC_d\)_{d\ge 3}$ composed of operators from $\U(\omega)$ with $\omega=\cO(\nicefrac{1}{c})$ ($\B(\alpha)$ or $\C(\alpha)$ with $\alpha = \frac{1}{1+\Omega(c)}$) such that $\cC_d(x)\in\R^d$ can be encoded with $c d$ bits for any $x\in\R^d$.
\end{itemize}
\end{theorem}

Thus, $\Theta(d)$ bits need to be transmitted in order to bound the variance by a constant. The same asymptotic bound, $\Theta(d)$ bits per node, on total communication holds for distributed mean estimation \cite{ZDJWM,RDME,Suresh2017}.

\subsection{Compressed learning algorithms}

\begin{table*}[t]
\begin{center}
\begin{small}
\begin{sc}
\begin{tabular}{lcc}
\toprule
Compressed Learning Algorithm & Objective Function & Iteration complexity  \\
\midrule
Compressed GD (CGD) \cite{bez20,khirirat2018distributed} & $L$-smooth, $\mu$-convex   & $ \tilde{\cO}\( \frac{\kappa}{1-\alpha} \),\; \tilde{\cO}\( (\omega+1) \kappa \)$ \vspace{2pt}\\
Accelerated CGD \cite{AccCGD} & $L$-smooth, $\mu$-convex   & $ \tilde{\cO}\( (\omega+1) \sqrt{\kappa} \)$ \vspace{2pt}\\
Accelerated CGD \cite{AccCGD} & $L$-smooth, convex   & $ \cO\( (\omega+1) \sqrt{\nicefrac{L}{\varepsilon}} \)$ \vspace{2pt}\\
Distributed CGD-DIANA \cite{MGTR,DIANA-VR}  & $L$-smooth, $\mu$-convex         & $\tilde{\cO}\( \omega + \frac{\omega\kappa}{n} + \kappa \)$ \vspace{2pt}\\
Distributed ACGD-DIANA \cite{AccCGD}  & $L$-smooth, $\mu$-convex         & $\tilde{\cO}\( \omega + \sqrt{(\nicefrac{\omega}{n}+\sqrt{\nicefrac{\omega}{n}})\omega\kappa} + \sqrt{\kappa} \)$ \vspace{2pt}\\
Quantized SGD (QSGD) \cite{qsgd} & $L$-smooth, convex   & $\cO\( \frac{\omega}{n}\frac{1}{\varepsilon^2} + \frac{L}{\varepsilon} \)$ \vspace{4pt}\\
Distributed Compressed SGD \cite{Cnat} & $L$-smooth, non-convex   & $\cO\( (\omega+1)\(\nicefrac{\omega}{n}+1\)\frac{L}{\varepsilon^2} \)$ \vspace{4pt}\\
Compressed SGD with \vspace{-2pt}\\ Error Feedback (EF-SGD) \cite{stich2019,bez20} & $L$-smooth, $\mu$-convex   & $\tilde{\cO}\( \frac{\kappa}{1-\alpha} + \frac{1}{\mu\varepsilon} \)$ \vspace{1pt}\\
Compressed EF-SGD \cite{karimireddy2019error} & $L$-smooth, non-convex   & $\cO\( \frac{L^2}{\varepsilon}\(\frac{1}{\varepsilon} + \frac{1}{(1-\alpha)^2}\) \)$ \vspace{1pt}\\
DoublSqueeze \cite{DoubleSqueeze2019}                 & smooth, non-convex      & $\cO\( \frac{1}{n\varepsilon^2} + \frac{1}{1-\alpha}\frac{1}{\varepsilon^{1.5}} + \frac{1}{\varepsilon} \)$ \\
\bottomrule
\end{tabular}
\end{sc}
\end{small}
\end{center}
\caption{Iteration complexities of various compressed learning algorithms with respect to the variance ($\omega$ or $\alpha$) of the compression operator. For smooth and strongly convex ($\mu$-convex with $\mu>0$) objectives $\kappa=\nicefrac{L}{\mu}$ indicates the condition number. $\tilde{\cO}$ hides logarithmic factor $\log\nicefrac{1}{\varepsilon}$, $n$ denotes the number of nodes, $\varepsilon$ is the desired convergence accuracy.}
\label{table:iter-complexity}
\vskip -0.1in
\end{table*}

To highlight the importance of investigating the communication-variance trade-off of compression operators, we present how these operators affect the performance of compressed learning algorithms. For the sake of simplicity, consider distributed {\em Compressed Gradient Descent (CGD)} with compression operator $\cC\in\U(\omega)$ solving the following smooth non-convex optimization problem
$$
\min_{x\in\R^d} f(x) \eqdef \frac{1}{n}\sum_{i=1}^n f_i(x),
$$
where $n$ is the number of nodes or machines available and $f_i(x)$ is the loss function corresponding to the data stored at node $i$. Hence, CGD algorithm iteratively performs the updates $x^{t+1} = x^t - \gamma_t g^t$ with unbiased gradient estimator
$$
g^t = \frac{1}{n}\sum_{i=1}^n g_i^t \eqdef \frac{1}{n}\sum_{i=1}^n \cC(\nabla f_i(x^t)).
$$
Using smoothness of the loss function $f(x)$, the expected loss is upper bounded as follows:
\begin{align*}
\E[f(x^{t+1}) | x^t]
&= \E_t[f(x^t - \gamma_t g^t)] \\
&\overset{L\text{-smoothness}}{\le} f(x^t) - \gamma_t\|\nabla f(x^t)\|^2 + \frac{L\gamma_t^2}{2}\E_t\[\|g^t\|^2\] \\
&= f(x^t) - \frac{2\gamma_t-L\gamma_t^2}{2}\|\nabla f(x^t)\|^2 + \frac{L\gamma_t^2}{2}\E_t\[\|g^t-\nabla f(x^t)\|^2\],
\end{align*}
where $L>0$ is the smoothness parameter. Now, the term that is affected by compression and slowing down the convergence is the last one, namely the variance of estimator $g^t$, which can be transform into
\begin{align*}
\E_t\[\|g^t-\nabla f(x^t)\|^2\]
&= \E_t\[\| \frac{1}{n}\textstyle\sum_{i=1}^n \(g_i^t-\nabla f_m(x^t)\)\|^2\] \\
&= \frac{1}{n^2}\textstyle\sum_{i=1}^n\E_t\[\|g_i^t-\nabla f_i(x^t)\|^2\] \overset{(\ref{class-unbiased})}{\le} \frac{\omega}{n^2}\textstyle\sum_{i=1}^n\|\nabla f_i(x^t)\|^2.
\end{align*}
Clearly, in case of no compression ($\omega=0$), this term vanishes. Thus, the slowdown caused by the compression operator $\cC\in\U(\omega)$ is controlled by its parameter $\omega$.

Similarly, for compression operators from $\B(\alpha)$ or $\C(\alpha)$, the parameter $\alpha$ controls the slowdown. Table \ref{table:iter-complexity} summarizes iteration complexities of various learning algorithms exploiting compressed communication and exposes the dependence of the variance ($\omega$ and $\alpha$) of compression operator. The conclusion from this discussion and from Table \ref{table:iter-complexity} is that to facilitate fast and communication-efficient training process, one needs to design compression operators minimizing both variance and number of encoding bits. This is the motivation of our work. Therefore, compression operators developed in this paper can be incorporated in any compressed learning algorithm, including all the ones in Table \ref{table:iter-complexity}.

\section{Worst-Case Analysis}\label{sec:wca}

We start our analysis of compression operators with respect to the number of encoding bits in the worst case. First, we show that the lower bound (\ref{up-alpha}) for the class $\B(\alpha)$ is asymptotically tight for any $\alpha\in(0,1)$. Then, we design an efficient compression operator from $\C(\alpha)$, {\em Sparse Dithering}, which is within a small constant factor of being optimal. Finally, we derive asymptotically tighter lower and upper bounds.

\subsection{Asymptotic tightness of the lower bound (\ref{up-alpha})}

First, we show that for any fixed $\alpha\in(0,1)$, the constant 1 in the lower bound (\ref{up-alpha}) is not improvable. We denote by $\sphere^d=\{x \in\mathbb{R}^d\ :\ \|x\|=1\}$ the unit sphere of $\mathbb{R}^d$.

\begin{theorem}\label{thm:tight-construction-alpha}
For any given $\alpha\in(0,1)$ and $d\ge 3$ there exists an
$\alpha$-contractive compression operator $\cC\colon\sphere^d\to\R^d$, such  that
\begin{equation}\label{tight-up-bound}
\alpha \, 4^{\nicefrac{b}{d}} \le \(1600 d^2\log d\)^{\nicefrac{2}{d}},
\end{equation}
where $b$ is the number of bits (in the worst case) needed to encode $\cC(x)\in\R^d$ for any unit vector $x\in\sphere^d$. In particular, for any $\alpha\in(0,1)$ and $\epsilon>0$ one can choose $d$ large enough such that compression operator $\cC$ satisfies
\begin{equation}\label{epsilon-optimal-bound}
\alpha \, 4^{\nicefrac{b}{d}} < 1 + \epsilon.
\end{equation}
\end{theorem}

\begin{remark}
Using covering results from \cite{Ilya2007} (see Theorem 1), the constant 1600 in (\ref{tight-up-bound}) can be reduced up to 2. Using tighter inequalities for the $\Gamma$ function, the term $d^2$ can be improved as well. However, these will not improve the inequality (\ref{epsilon-optimal-bound}). Notice that the right hand side of (\ref{tight-up-bound}) approaches  1 quickly; for $d=10^3$ it is $\approx 1.047$.
\end{remark}

Note that the compression operator in this theorem acts on $\sphere^d$, not $\R^d$. However, allocating an additional constant amount of bits for the norm $\|x\|$ (say $31$ bits in {\em float32} format), we can extend the domain of compression operators without hurting the asymptotic tightness. Thus, the lower bound (\ref{up-alpha}) is asymptotically tight for the class $\B(\alpha)$.

Although the construction of this theorem yields an optimal contractive operator, it is infeasible to apply in high dimensions.

\subsection{New Compressor: Sparse Dithering (SD)}\label{sec:sd}

With the aim of constructing both optimal and efficient compression operators, we introduce a new compression scheme--{\em Sparse Dithering (SD)}--which is efficient in high dimension and nearly optimal. In some sense, SD can viewed as an effective combination of Top-$k$ sparsification \cite{Alistarh-SparsGradMethods2018} and random dithering with uniform levels \cite{qsgd}. The essential novelty is the encoding scheme and better upper bound on the number of communicated bits. In this section, we present a deterministic and hence biased version of SD.

{\bf Construction and variance bound.} To compress a given nonzero vector $x\in\R^d$, we first compress the normalized vector $u = x/\|x\|\in\sphere^d$ and then rescale it. To quantize the coordinates of the unit vector $u$, we apply dithering with levels $2k_i h,\; k_i\ge 0$, where $h = \sqrt{\nu/d}$ is the half-step and $\nu>0$ is a free parameter. For each coordinate $u_i, i\in[d]$ we choose the nearest level so that $||u_i| - 2k_ih|\le h$. Letting $\hat{u}_i = \sign(u_i) \, 2k_ih$ we have $|u_i-\hat{u}_i|\le h$ for all $i\in[d]$. Therefore,
$$
\|u - \hat{u}\|^2 = \sum_{i=1}^d (u_i-\hat{u}_i)^2 \le d h^2 = \nu.
$$

Note that, after applying the scaling factor $\|x\|$, this gives a compression with variance at most $\nu$. However, $\|x\|$ is not always the best option. Specifically, we can choose the scaling factor $\gamma>0$ so to minimize the variance $\|x - \gamma\hat{u}\|^2$, which yields the optimal factor $\gamma^* = \frac{\<x, \hat{u}\>}{\|\hat{u}\|^2}$ with the optimal variance of $\|x-\gamma^*\hat{u}\|^2 = \sin^2\varphi\,\|x\|^2$, where $\varphi\in[0,\nicefrac{\pi}{2}]$ is the angle\footnote{in case of $\cC(x)=0$ we let $\varphi=\nicefrac{\pi}{2}$.} between $x$ and $\hat{u}$. Hence, defining the compression operator as $\cC(x) = \gamma^*\hat{u}$, we have the following bound on the variance:
$$
\|\cC(x) - x\|^2 \le \min\(\nu, \sin^2\varphi\) \|x\|^2.
$$

{\bf Encoding scheme.} We now describe the corresponding encoding scheme into a sequence of bits.
With the following notations:
\begin{equation*}
\gamma \eqdef 2h\gamma^*\in\R_+, \quad k \eqdef (k_i)_{i=1}^d\in\N_+^d, \quad s \eqdef \(\sign(u_ik_i)\)_{i=1}^d\in\{-1,0,1\}^d,
\end{equation*}
the
compression operator can be written as $\cC(x) = \gamma^*\hat{u} = 2h\gamma^* \, \sign(u) \, k = \gamma \, s \, k$. So, we need to encode the triple $(\gamma, s, k)$. As $\gamma\in\R_+$, we need only \underline{$31$} bits for the scaling factor. Next we encode $s$. Let $$n_0 \eqdef |\{i\in[d] \colon s_i=0\}| = |\{i\in[d] \colon k_i=0\}|$$ be the number of coordinates $u_i$ that are compressed to 0.
To communicate $s$, we first send the locations of those $n_0$ coordinates and then \underline{$d-n_0$} bits for the values $\pm 1$. Sending $n_0$ positions can be done by sending \underline{$\log d$} bits representing the number $n_0$, afterwards sending \underline{$\log\binom{d}{n_0}$} bits for the positions.
Finally, it remains to encode $k$, for which we only need to send nonzero entries, since the positions of $k_i=0$ are already encoded. We encode $k_i\ge 1$ with $k_i$ bits: $k_i-1$ ones followed by a zero. Hence, encoding $k$ requires \underline{$\sum k_i$} bits.

A theoretical upper bound on the total number of bits for any choice of parameter $\nu>0$ is given in the Appendix. Below, we highlight one special case of $\nu=\nicefrac{1}{10}$.

\begin{theorem}\label{thm:DSD}
Deterministic SD compression operator with parameter $\nu=\nicefrac{1}{10}$ belongs to $\C(\nicefrac{1}{10})$ communicating $30 + \log d + 3.35d$ bits at most. In addition, ignoring $30+\log d$ negligible bits, SD is within a factor of
$$
\log_4(\alpha\, 4^{\nicefrac{b}{d}}) = \log_4 \(\frac{1}{10}4^{3.35} \) \approx 1.69
$$
of optimality; that is,  at most $1.69d$ more bits are sent in comparison to optimal compression with the same normalized variance $\nicefrac{1}{10}$.
\end{theorem}

\subsection{Tighter bounds on minimal communication}

We first look into the tightness of (\ref{up-alpha}) when the normalized variance $\alpha$ approaches $1$. In particular, for $\alpha = 1-\frac{1}{d}$ the lower bound (\ref{up-alpha}) implies that the number of bits $b$ is lower bounded by some constant. However, the following holds:

\begin{theorem}\label{thm:logd-bits}
For any compression operator from $\B(\alpha)$, with $\alpha\in(0,1)$, at least $\log{d}$ bits are needed.
\end{theorem}

As briefly mentioned before, the lower bound (\ref{up-alpha}) is tight up to a $\cO(\log d)$ additive error term. Here we perform a deeper analysis of the same lower bound.

\begin{definition}
For a fixed $\alpha\in(0,1)$ and dimension $d$, consider compression operators $\cC\in\B(\alpha)$ with underlying encoder $E$, decoder $D$ and define $b^*(\alpha,d)$ as the minimum number of bits in the worst case:
$$
b^*(\alpha,d) = \min_{\cC\in\B(\alpha)} \max_{\|x\|=1} |E(x)|.
$$
In other words, for any compression operator from $\B(\alpha)$ there exists a unit vector that cannot be encoded into less than $b^*(\alpha,d)$ bits and it shows the least amount of bits with such property.
\end{definition}

Combining lower bound (\ref{up-alpha}) with (\ref{tight-up-bound}) of Theorem \ref{thm:tight-construction-alpha}, yields
$$
b^*(\alpha,d) = \frac{1}{2}\log\frac{1}{\alpha} + e \; \text{ with error term } \; 0\le e = \cO(\log d).
$$
Denoting $P(\alpha,d) \eqdef \frac{1}{2}I_{\alpha}\(\frac{d-1}{2}, \frac{1}{2}\)\in\(0,\frac{1}{2}\)$, where $I_{\alpha}$ is the regularized incomplete beta function, we show tighter asymptotic behavior:

\begin{theorem}\label{thm:tighter-bounds}
With error term $|e| \le \frac{1}{2}\log{\log{d}} + \cO(1)$,
$$
b^*(\alpha, d) = -\log{P(\alpha, d)} + \log{d} + \frac{1}{2}\log{\log{d}} + e.
$$
\end{theorem}

\section{Average-Case Analysis}\label{sec:aca}

Now we switch to the average-case analysis for the class $\C(\alpha)$. First, we prove a lower bound for communicated bits in expectation. Then we analyze the randomized version of Sparse Dithering, which, having better theoretical guarantees than random dithering, is suboptimal in this analysis. Finally, we present a new compression operator from $\C(\alpha)$, {\em Spherical Compression}, which is provably optimal.

\subsection{Lower bound on average communication}

In this section, we consider compression operators from $\C(\alpha)$ and investigate the trade-off between normalized variance $\alpha$ and expected number of bits $$B=\sup_{\|x\|=1}\E_{\cC}\[|E(x)|\].$$ In other words, we study the trade-off for strictly $\alpha$-contractive operators that encode any unit vector with no more than $B$ bits in expectation. In such a setting, we show the following lower bound on $B$.

\begin{theorem}\label{thm:tighter-bounds-2}
Let $\cC\in\C(\alpha)$ be a compression operator such that $\cC(x)\in\R^d$ can be transferred with $B$ bits in expectation for any unit vector $x\in\sphere^d$. Then $-\log P(\alpha, d) \le B$.
\end{theorem}

\subsection{Randomized version of Sparse Dithering}

Here we randomize Sparse Dithering to make it unbiased and estimate the number of encoding bits it needs in expectation. First we decompose the to-be-compressed vector $x\in\R^d$ into the magnitude and unit direction $u=x/\|x\|$ as before. To randomize the scheme, each coordinate $u_i$ gets rounded to one of the two nearest neighbors, so as to preserve unbiasedness; that is, if $2k_ih \le |u_i| \le 2(k_i+1)h$ for some $k_i\ge0$, then $\hat{u}_i = \sign(u_i) 2\hat{k}_ih$ where
\begin{equation*}
\hat{k}_i =
\begin{cases}
    k_i     & \text{ with prob. }\  \frac{2(k_i+1)h-|u_i|}{2h}\\ 
    k_i+1   & \text{ with prob. }\  \frac{|u_i| - 2k_ih}{2h}.\\
\end{cases}
\end{equation*}
Clearly, $\E\[\hat{u}\] = u$ and defining $\cC(x) = \|x\|\hat{u}$, we maintain unbiasedness $\E\[\cC(x)\]=x$. The encoding scheme is the same as in the deterministic case. Upper bounding the expected number of bits and the variance, we obtain:

\begin{theorem}\label{thm:RSD}
Randomized SD compression with parameter $\nu=\omega$ belongs to $\U(\omega)$, communicating at most
$$
30 + \log d + \(\log 3 + \frac{1}{2\sqrt{\omega}}\)d
$$
bits in expectation. In particular, with $\omega = \nicefrac{1}{4}$ variance (ignoring $30+\log{d}$ negligible factors), it uses $\(1+\log 3\)d\approx 2.6d$ bits in each iteration (about $12$ times less than full precision case) and forces up to $1+\omega=\nicefrac{5}{4}$ times more iterations, leading to $\approx 9.9$ times bandwidth savings.
\end{theorem}

As mentioned earlier, SD is similar to random dithering with uniform levels, namely with $\sqrt{d}$ levels. However, with a different parametrization $\nu$ and better encoding strategy, SD provides better theoretical guarantees. Indeed, random dithering with $\sqrt{d}$ levels communicates $\approx 2.8 d$ bits in expectation and requires $1+\omega=2$ times more iterations, resulting in a factor of  $\approx 5.7$ in bandwidth saving (see Theorem 3.2 and Corollary 3.3 of \cite{qsgd}). For more comparisons on bandwidth savings see Table 2 
in the Appendix.

\subsection{New Compressor: Spherical Compression (SC)}

It can be shown that randomized SD compression discussed in the previous section is suboptimal with respect to the lower bound of Theorem \ref{thm:tighter-bounds-2}. Here we provide a simple compression operator--{\em Spherical Compression (SC)}--that achieves this lower bound with less than $3$ overhead bits.

{\bf Construction and variance bound.}
As before, we transmit the magnitude and direction separately. For a given unit vector $x\in\sphere^d$, SC  generates a sequence $(x^t)_{t=1}^T$ of i.i.d. points with $\|x^t\|^2=1-\alpha$, and terminates once $\|x^T-x\|^2 \le \alpha$ for some $T\ge 1$. The last generated point $x^T$ is the compressed version of $x$ we need to communicate, that is $\cC(x)=x^T$. It follows directly from this construction that $\cC\in\C(\alpha)$.

{\bf Encoding scheme.}
The crucial part of the encoding scheme is that it is enough to communicate only $T$. Indeed, the communication process is the following. Importantly, the emitter and receiver have agreed on using the same random seed for generating i.i.d. points $(x^t)$, before the compression of any vector is performed. 

Then, upon receiving the number of trials $T$, the decoder can reproduce the same sequence $x^1, x^2, \dots, x^T$ and recover $x^T$. Consequently, it remains to encode the random integer $T$ into a binary code.

{\bf Upper bound on $B$.}
First we show that $T$ follows a geometric distribution with parameter $p=P(\alpha, d)$. Indeed, $T$ can be viewed as the number of trials before the first success happens after a series of failures.

\begin{figure}[t!]
\begin{center}
    \includegraphics[scale=0.53]{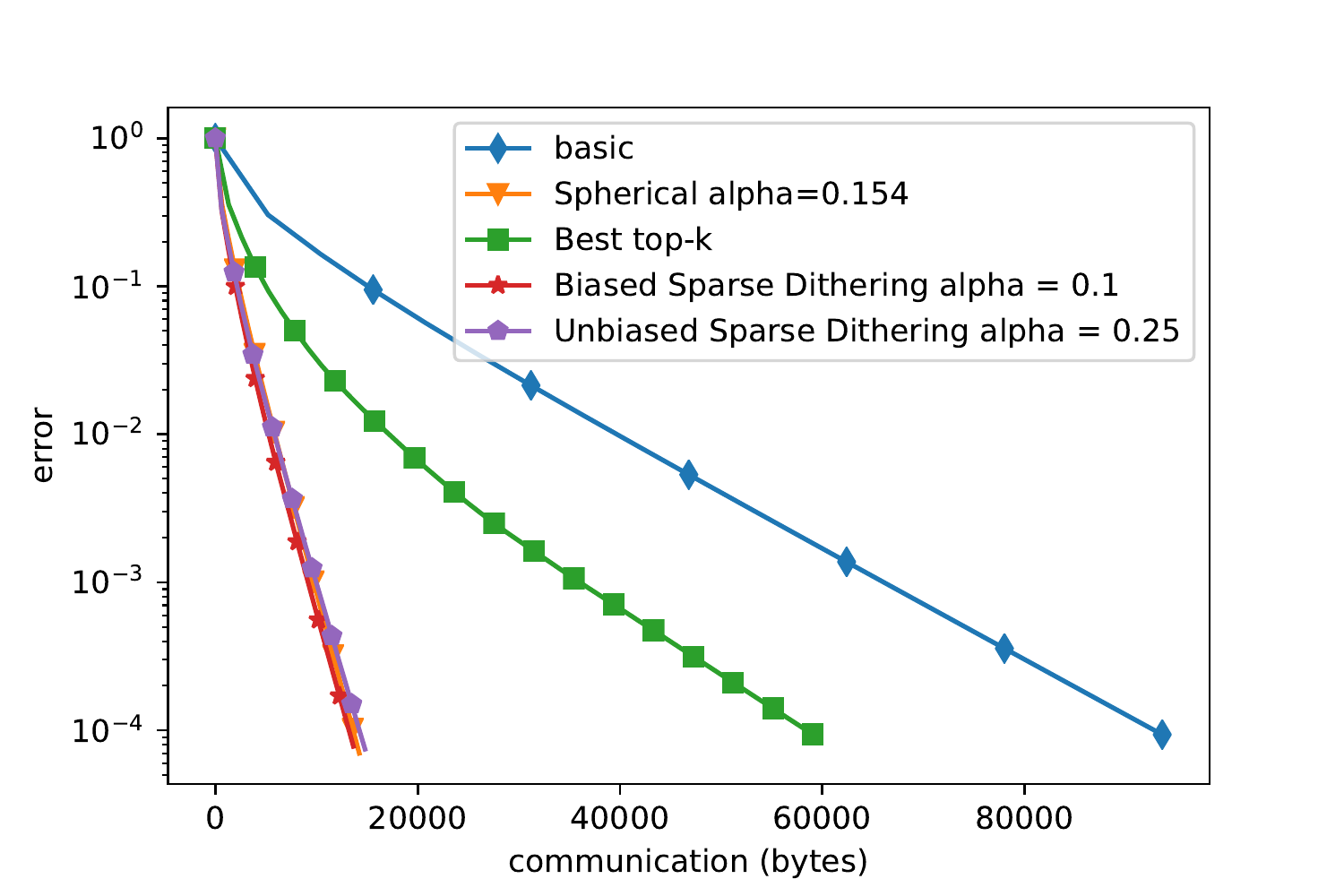}
    \includegraphics[scale=0.53]{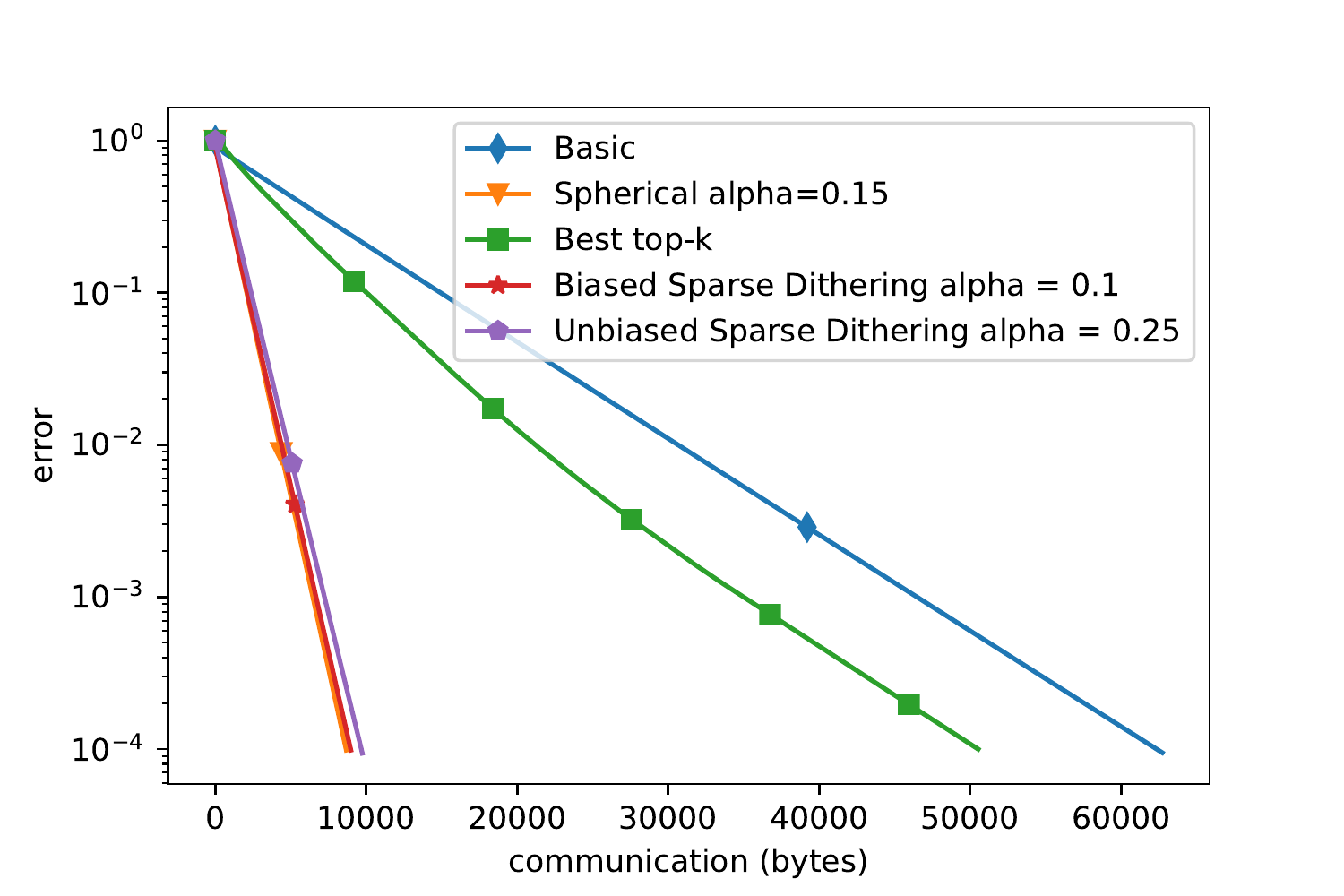}
    \includegraphics[scale=0.53]{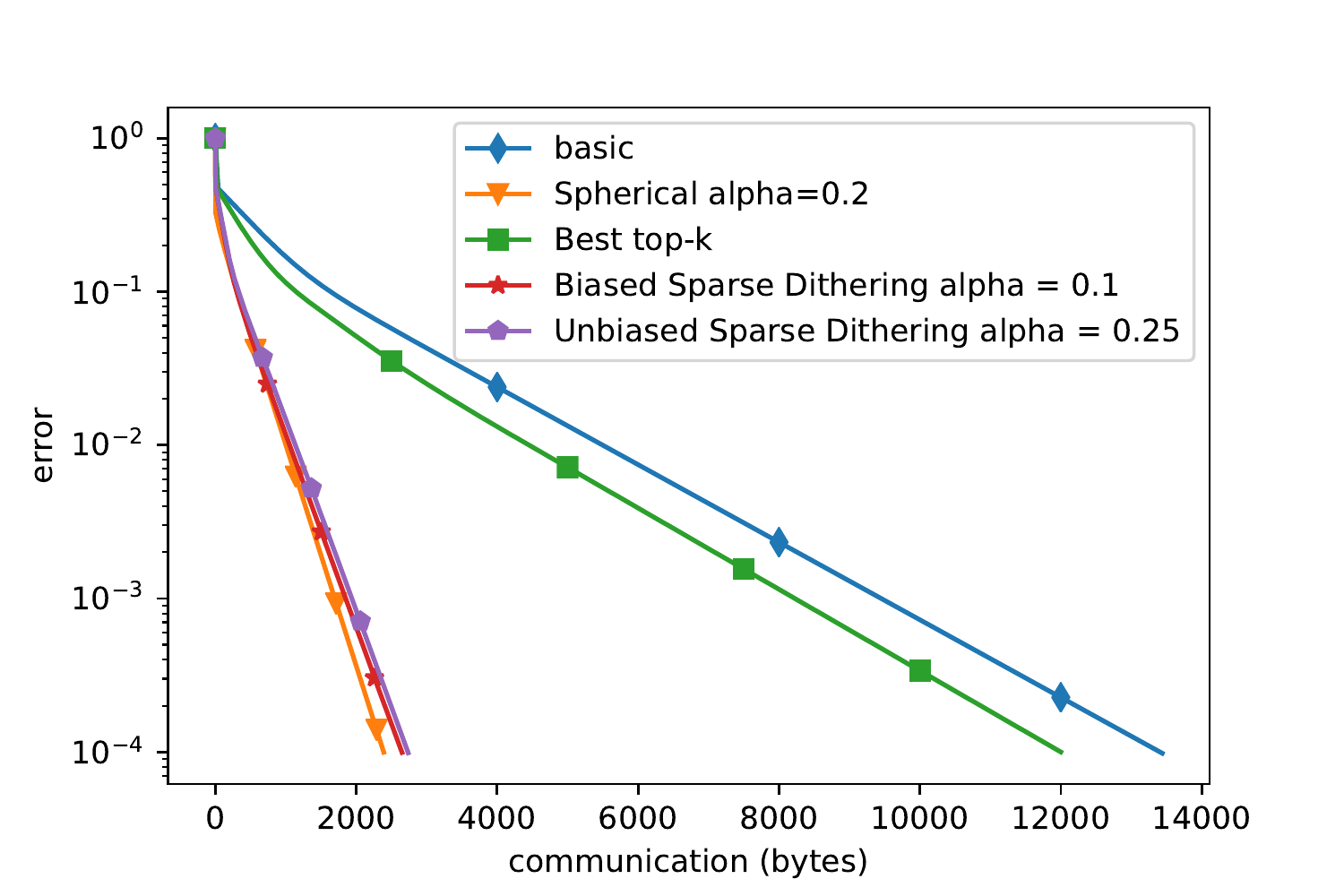}
    \includegraphics[scale=0.53]{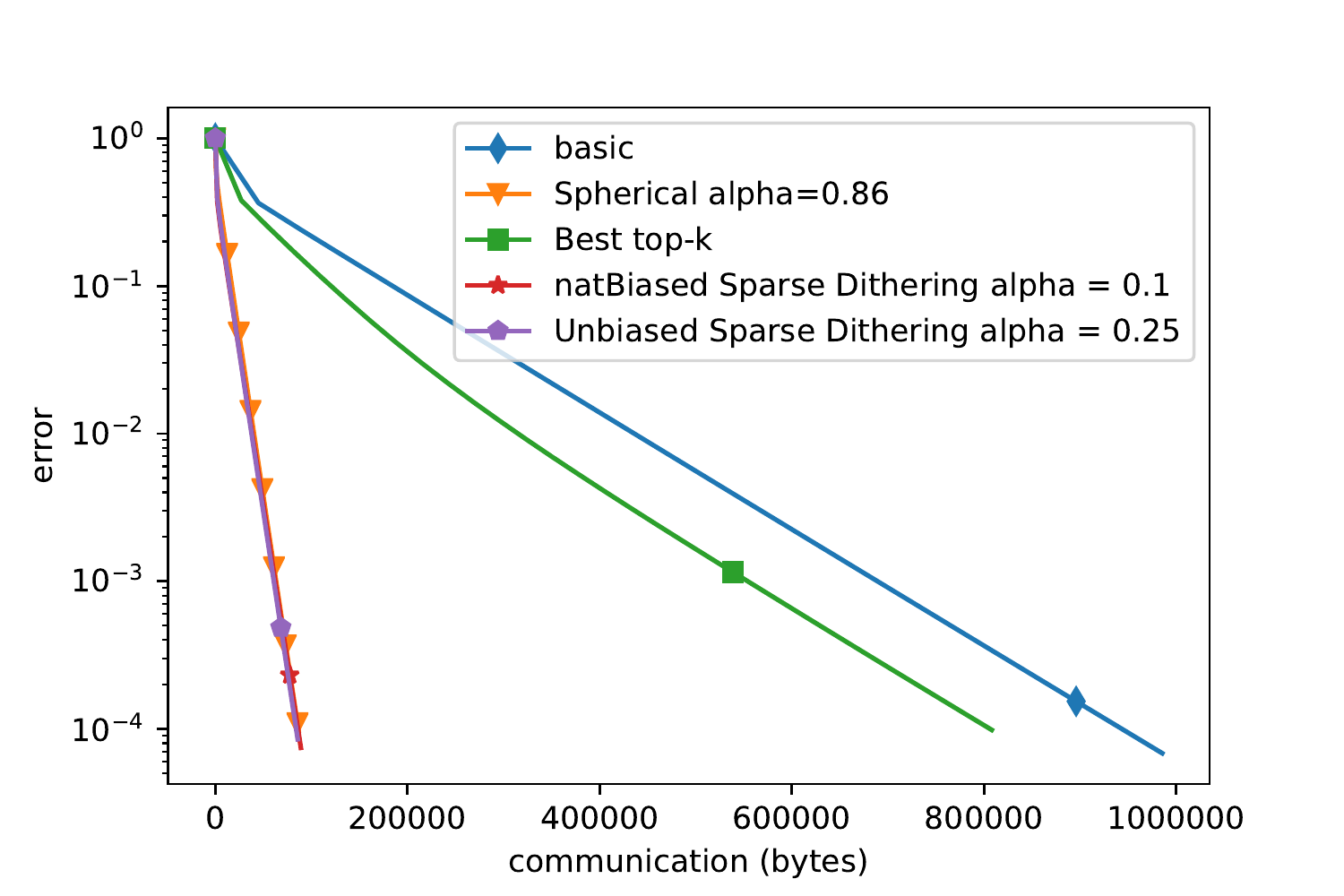}
    \includegraphics[scale=0.53]{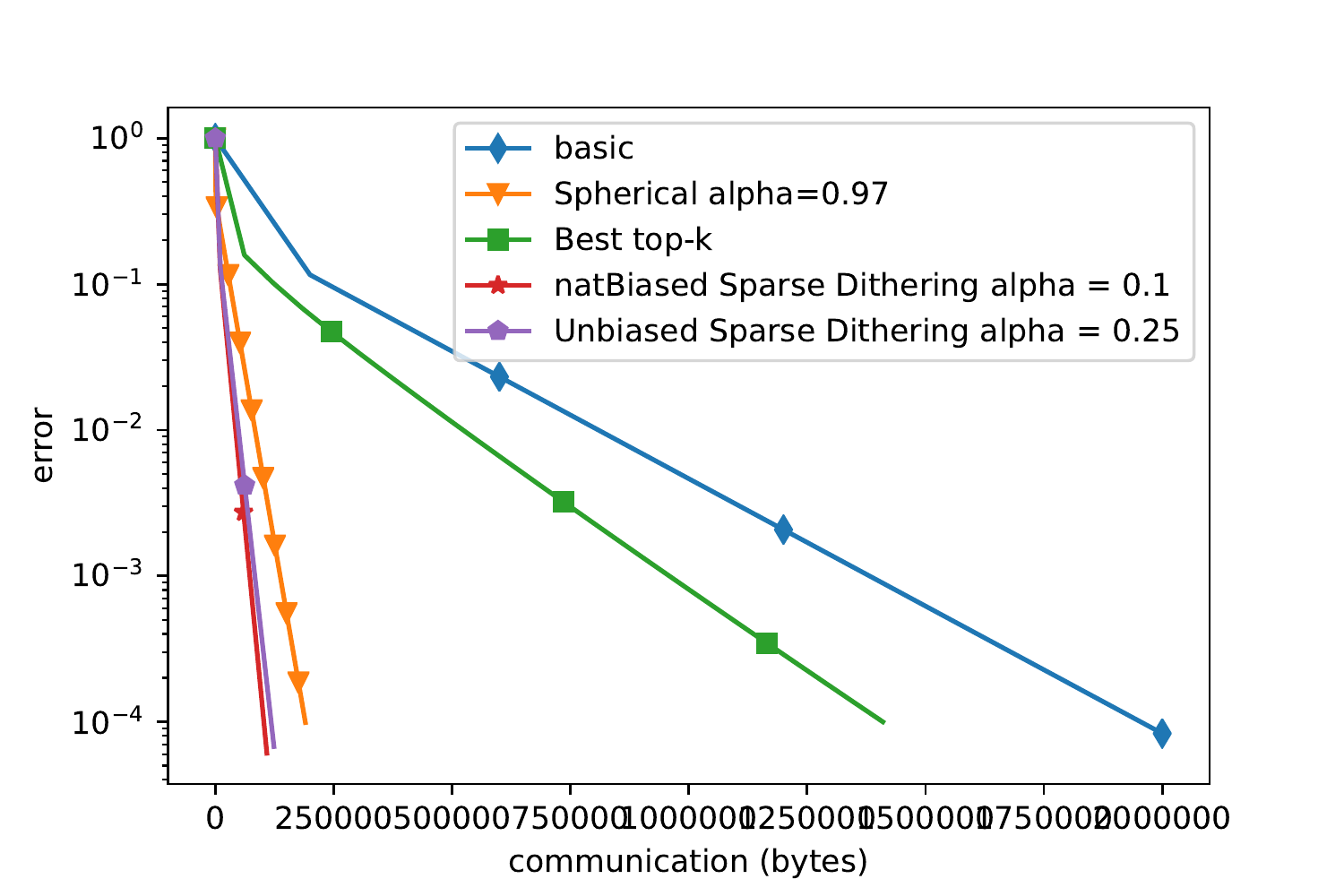}
\end{center} 
\caption{The first two plots correspond to ridge regression (Housing, Bodyfat datasets), while the next three plots correspond to regularized logistic regression (Breast Cancer, Madelon, Mushrooms datasets). This shows convergence as a function of total communication (in bytes), for various selected compression operators.}
\label{fig:1}
\end{figure}

\begin{figure}[t!]
\begin{center}
    \includegraphics[scale=0.53]{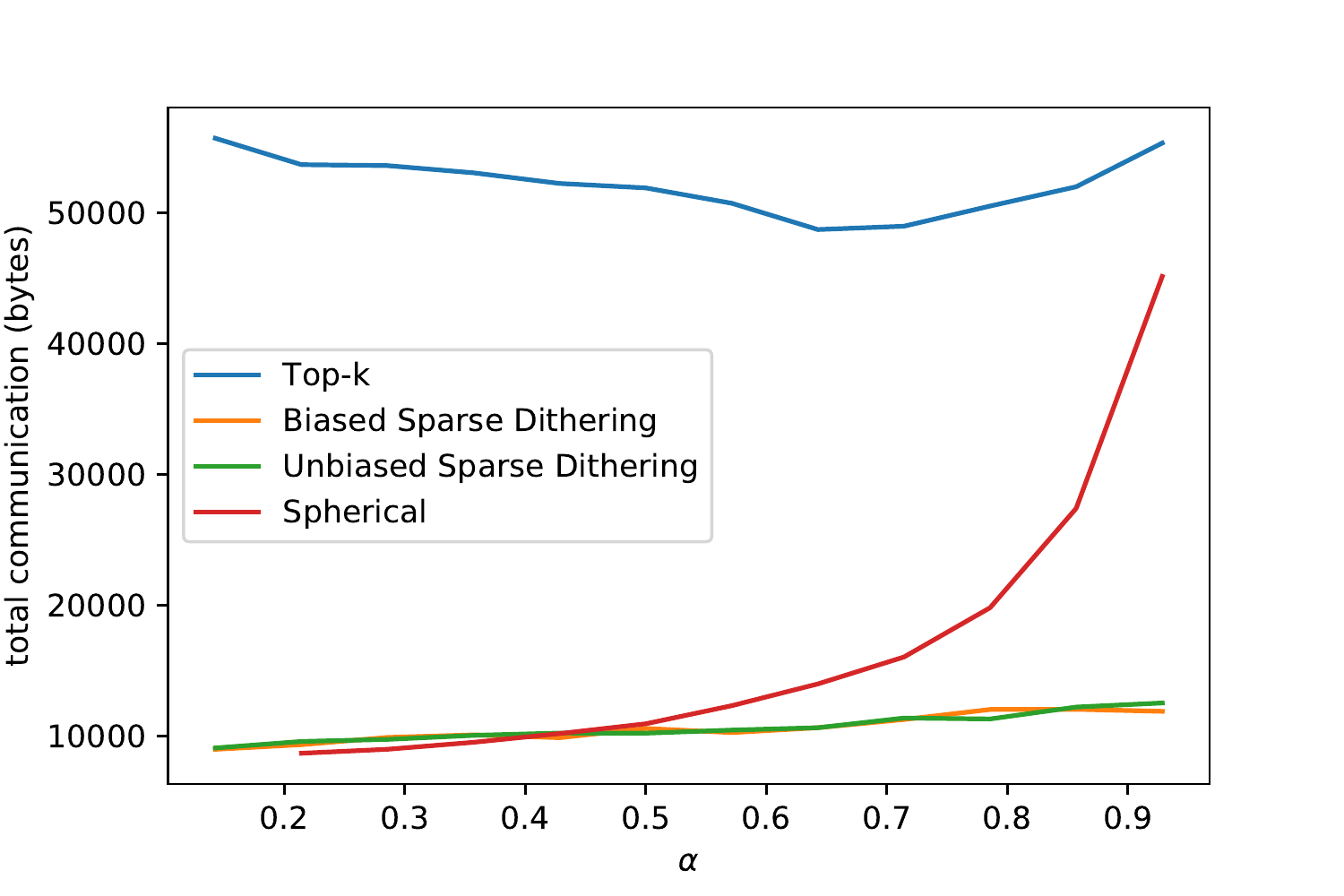}
    \includegraphics[scale=0.53]{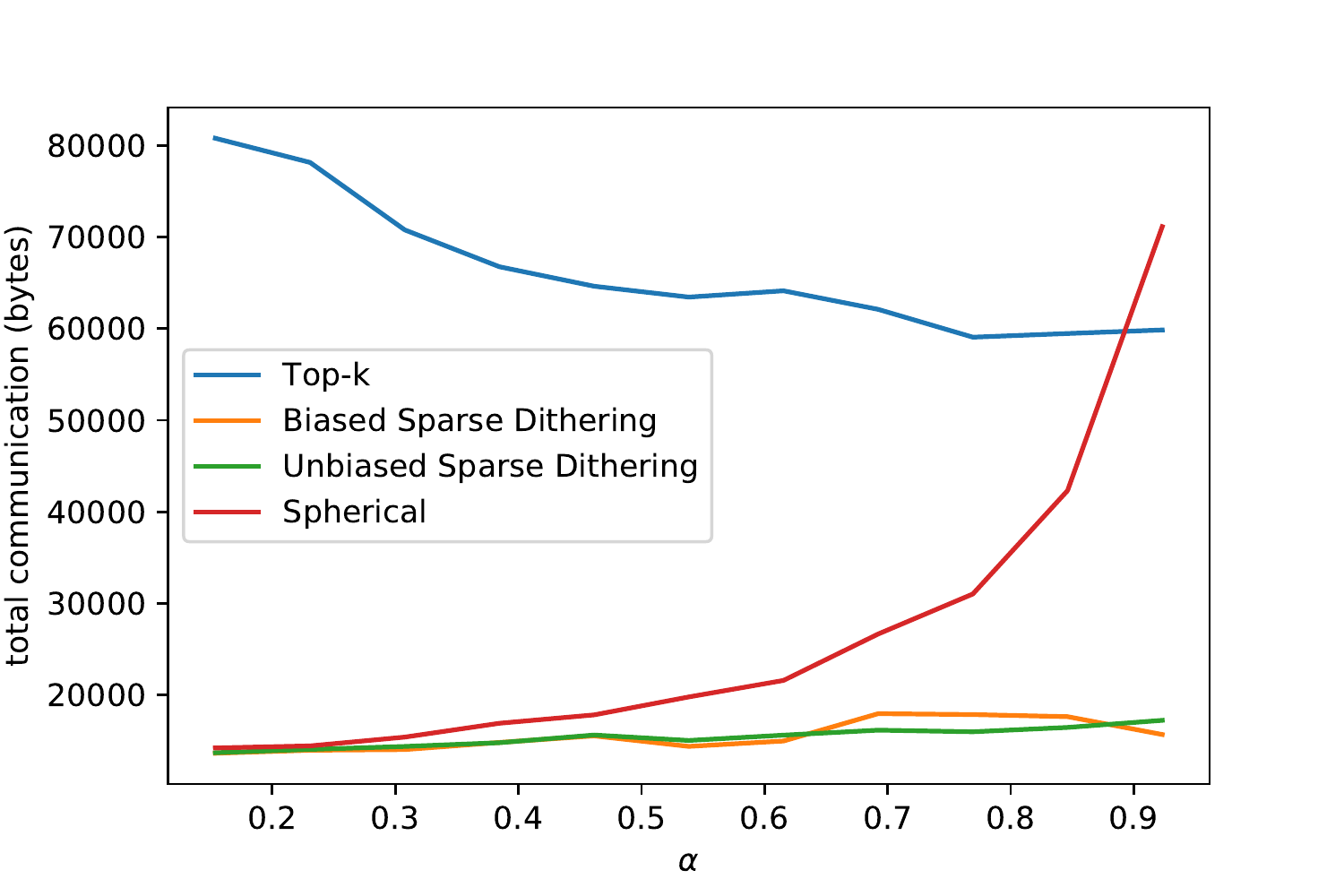}
    \includegraphics[scale=0.53]{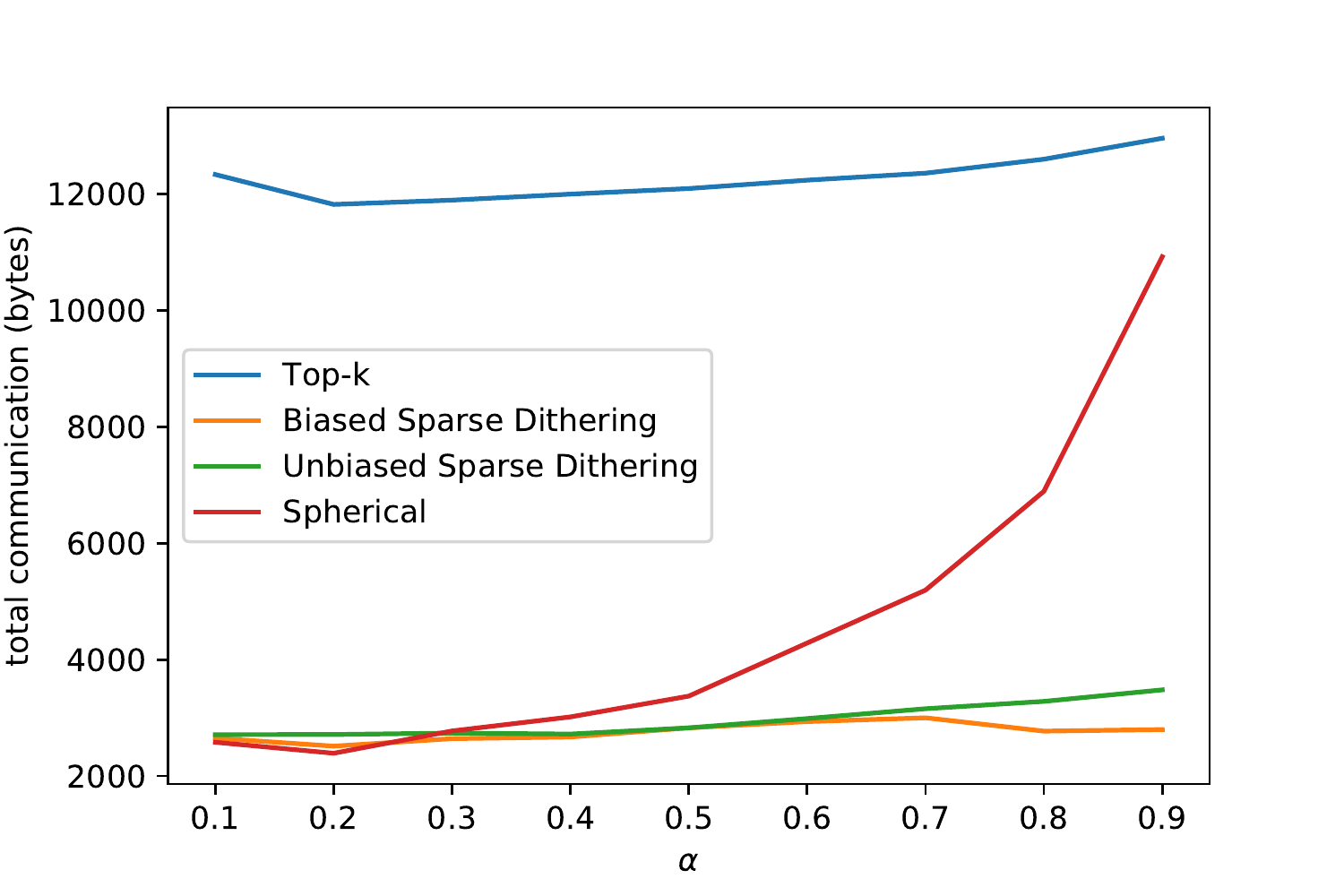}
    \includegraphics[scale=0.53]{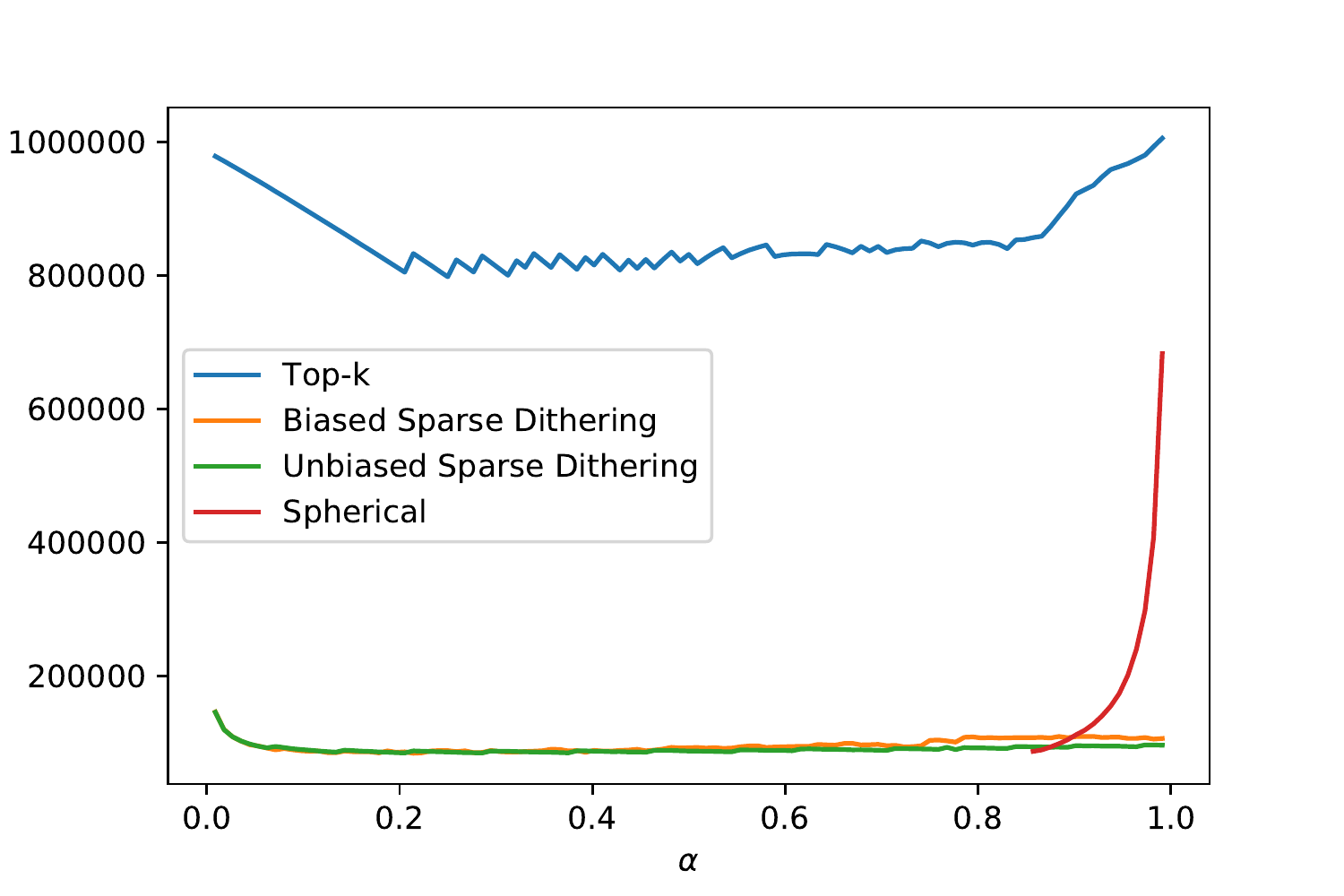}
    \includegraphics[scale=0.53]{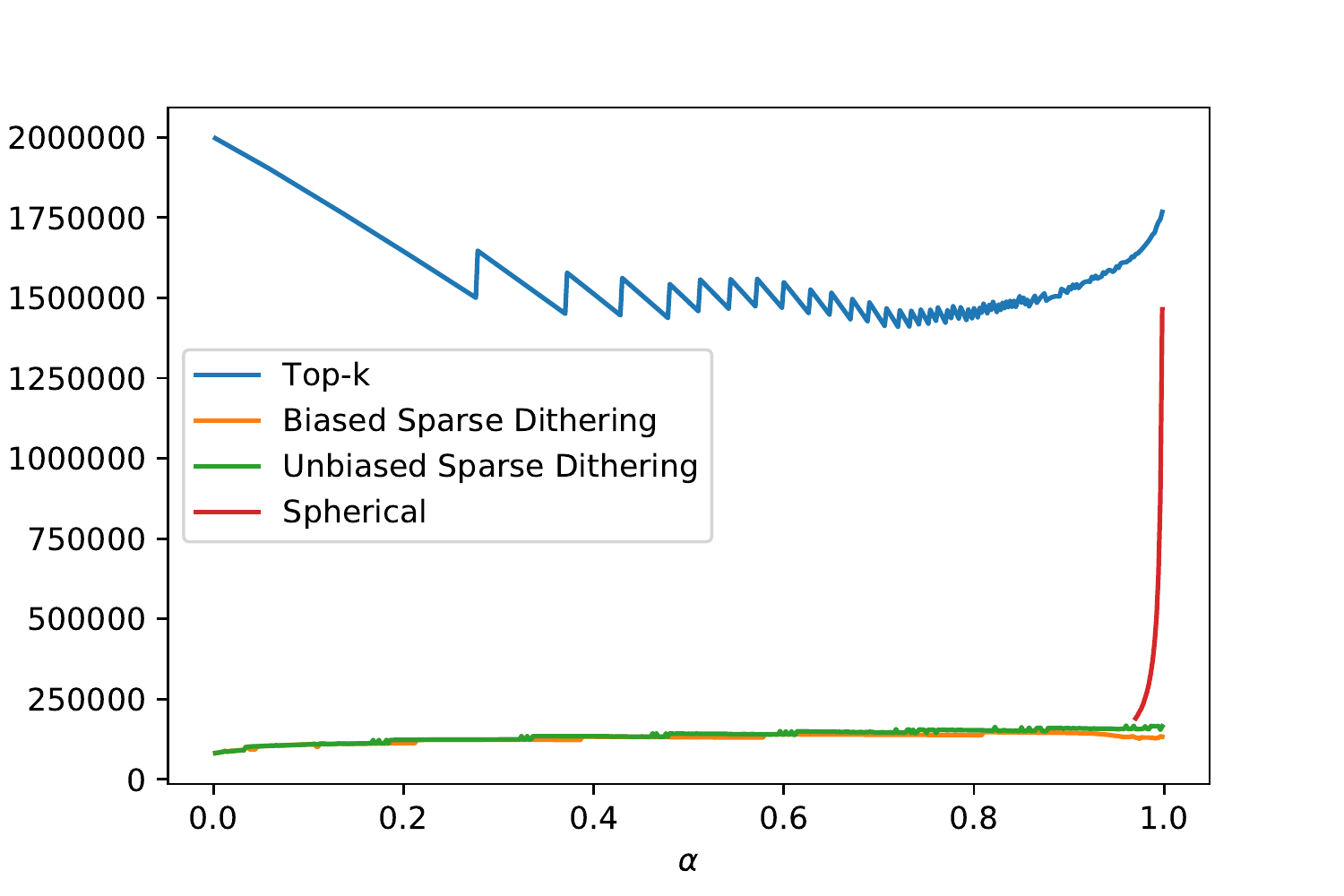}
\end{center} 
\caption{The first two plots correspond to ridge regression (Housing, Bodyfat datasets), while the next three plots correspond to regularized logistic regression (Breast Cancer, Madelon, Mushrooms datasets).  This shows total communication needed to achieve $\epsilon = 10^{-4}$ as a function of $\alpha$ for various operators. For Top-$k$, we set $\alpha = 1-\frac{k}{d}$, as predicted by the theory.}\label{fig:2}
\end{figure}


In our case, trials correspond to generating i.i.d. points $x^t$ and success means $x\in C^d(x^t,\sqrt{\alpha})$ which happens with probability $P(\alpha, d)\in(0,\nicefrac{1}{2})$. Therefore, the expected number of points $x^t$ we need to generate until we get into $\alpha$-vicinity of the initial point $x$ is $\E\[T\] = \nicefrac{1}{p}=\nicefrac{1}{P(\alpha,d)} > 2$.
Next, we encode $T$ with the Golomb--Rice coding scheme, which is known to be optimal for geometric distributions. Define integer $m\ge 0$ from $\nicefrac{1}{2p} \le 2^m < \nicefrac{1}{p}$ and decompose $T$ as $T = 2^m q + r$ with $q\ge 0,\; 0\le r < 2^m$. The quotient $q$ is encoded with unary coding as a string of $q$ zeros followed by a 1. The remainder $r$ is communicated with exactly $m$ bits using truncated binary coding. There is no need to send the value of $m$ as it can be computed from $p$, which depends only on $\alpha$ and $d$. Hence, the total number of bits to encode $T$ is no more than $q+m+1$. Note that $m < \log{\nicefrac{1}{p}} = -\log{p}$ is fixed, while $q$ depends on $T$ and $q\le 2^{-m}T$. Hence,
\begin{equation*}
B = \E\[q+m+1\] < 2^{-m}\E\[T\] - \log{p} + 1 = \frac{1}{2^m p} - \log{p} + 1\le - \log{p} + 3,
\end{equation*}
which implies:

\begin{theorem}\label{thm:SC}
In the average-case analysis, Spherical Compression is optimal up to $3$ extra bits; that is, it communicates $B < -\log P(\alpha, d) + 3$ bits in expectation.
\end{theorem}

\begin{remark}
It is worth mentioning that the above compression operator satisfies $\|\cC(x) - x\|^2 \le \alpha\|x\|^2$ in the worst case, not in expectation. Moreover, because of the symmetry of spheres and caps $C^d(x^t,\sqrt{\alpha})$, it can be seen from the construction that $\E\[\cC(x)\]$ points to the same direction as the initial vector $x$. Thus, with an appropriate (fixed) scaling factor, it can be made unbiased as well.
\end{remark}

\section{Experiments}

\subsection{Setting}
We consider both $l_2-$regularized logistic regression and ridge regression. In both cases, we use regularizing coefficient $\lambda = \frac{1}{n}$. We run this on multiple datasets, and show that our compression methods provide significant savings in communication (measured in bytes). The algorithm we used is Compressed Gradient Descent, which consists in iterating $$x^{t+1} = x^t- \frac{1}{L} \cC\(\nabla{f(x^t)}\),$$ where $f$ is the loss function, $L$ is the smoothness constant of $f$ computed explicitly. We stop the process whenever $\frac{\|x^t-x^\star\|^2}{\|x^0-x^\star\|^2} \le 10^{-4}$, where $x^\star$ is the minimizer of $f$ and is computed beforehand for all problems. \medskip

\subsection{Communication versus Convergence}
In this experiment, we look at convergence, measured as $\frac{\|x^t-x^\star\|^2}{\|x^0-x^\star\|^2}$ with respect to the number of bits communicated, for various compression operators. As shown in Figure \ref{fig:1},  our compression operators significantly outperform all  other operators. It is important that we compare our methods with the benchmark `Basic', which sends a 32-bits float for every element in the gradient, sending a total of $32d$ bits at every iteration. In addition, we run these experiments with Top-$k$  for all $1 \le k \le d$ and pick the best representative in the comparison, naming it `Best Top$-k$'. \medskip



\subsection{Total Communication as a Function of $\alpha$}
Here, we let $\alpha = 1-\nicefrac{k}{d}$ vary and in Figure \ref{fig:2}  we show the total number of bits communicated before converging to $\epsilon$-accuracy, with $\epsilon = 10^{-4}$. This clearly shows the superiority of our methods. It is important to note that Sparse Dithering can beat the optimal Spherical Compression, because it is unbiased, so it requires significantly less iterations. It's important to note that for $1 \le k \le d$, we plot Top-$k$ at $\alpha = 1-\nicefrac{k}{d}$.

\begin{figure}[t!]
\begin{center}
    \includegraphics[scale=0.53]{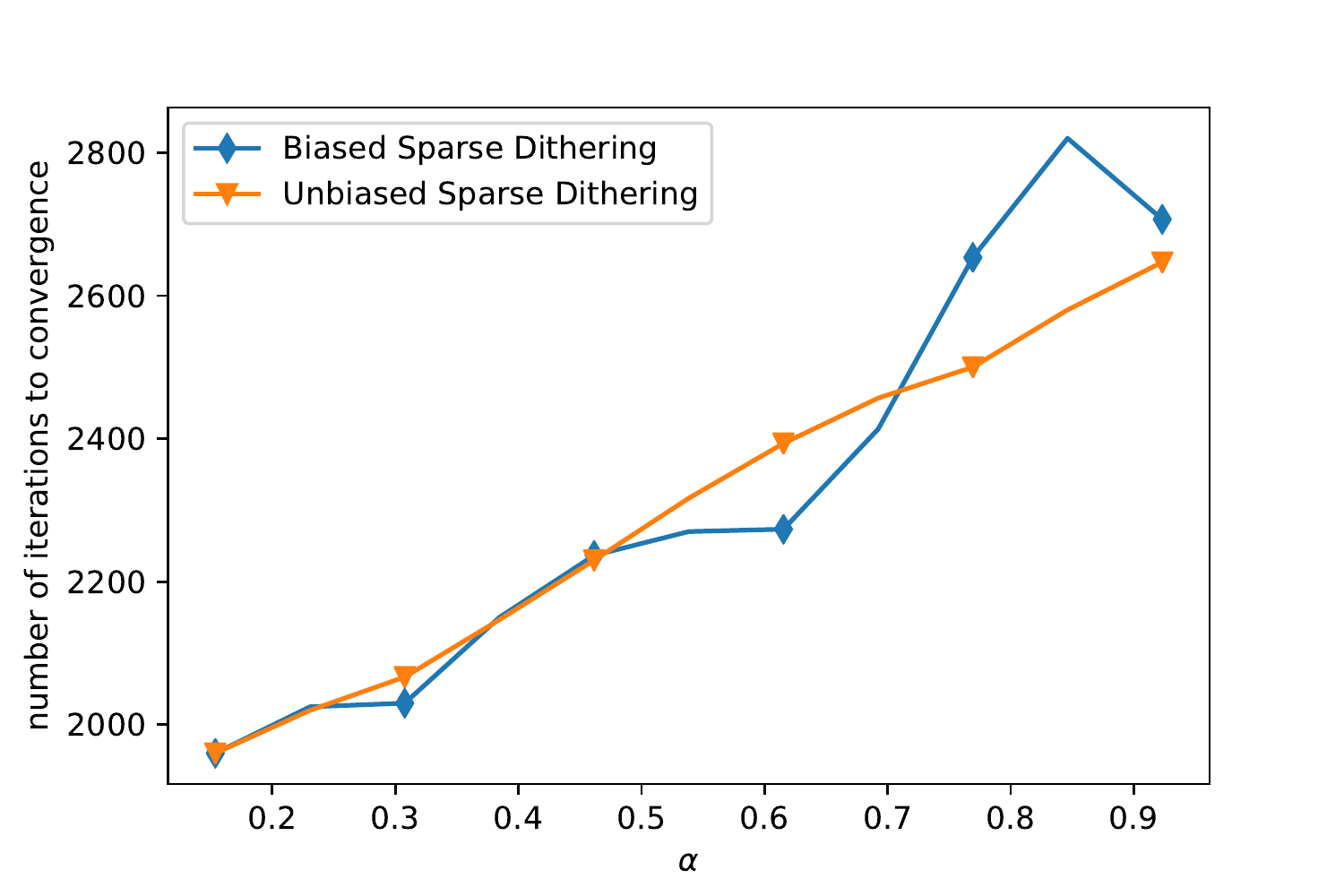}
    \includegraphics[scale=0.53]{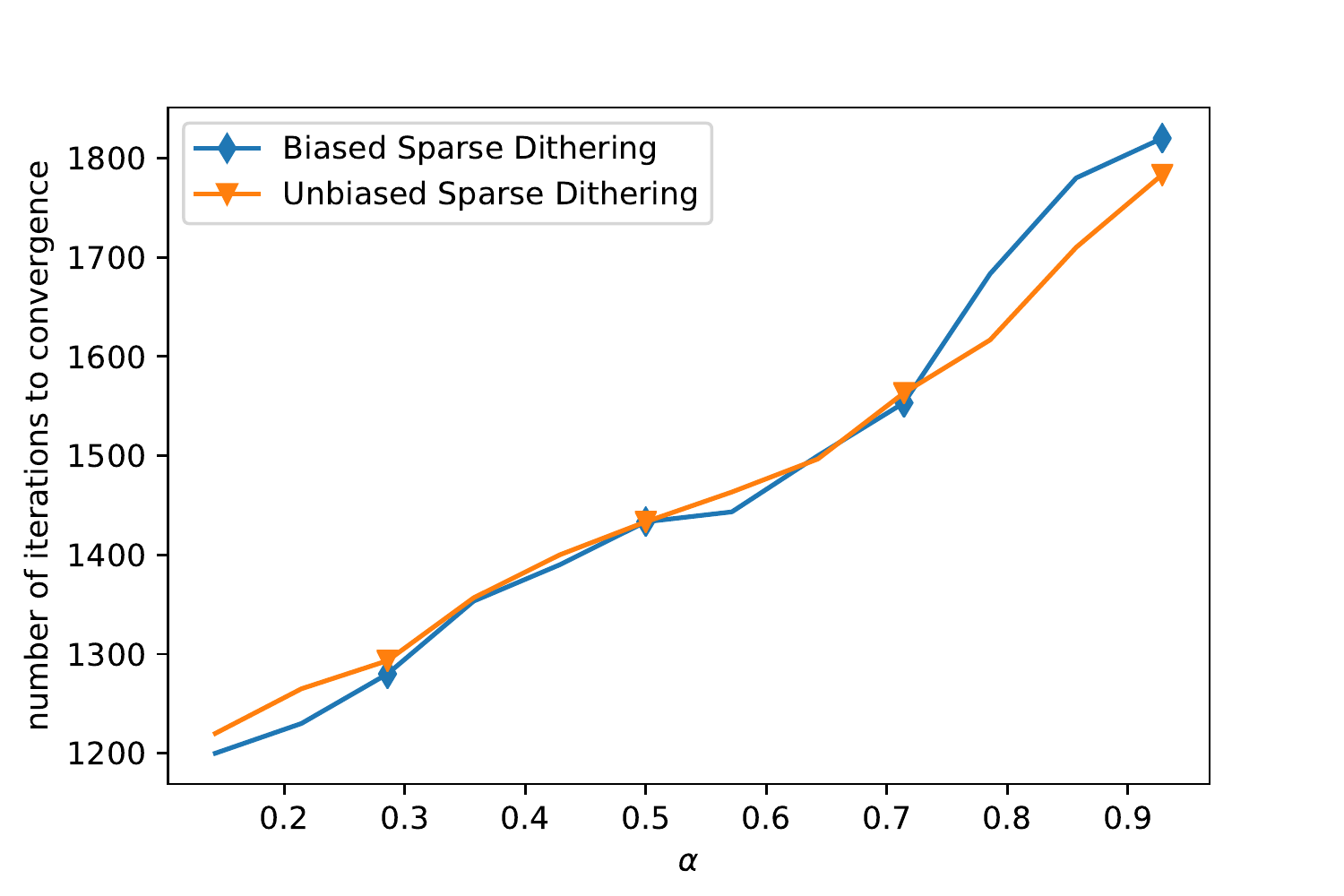}
    \includegraphics[scale=0.53]{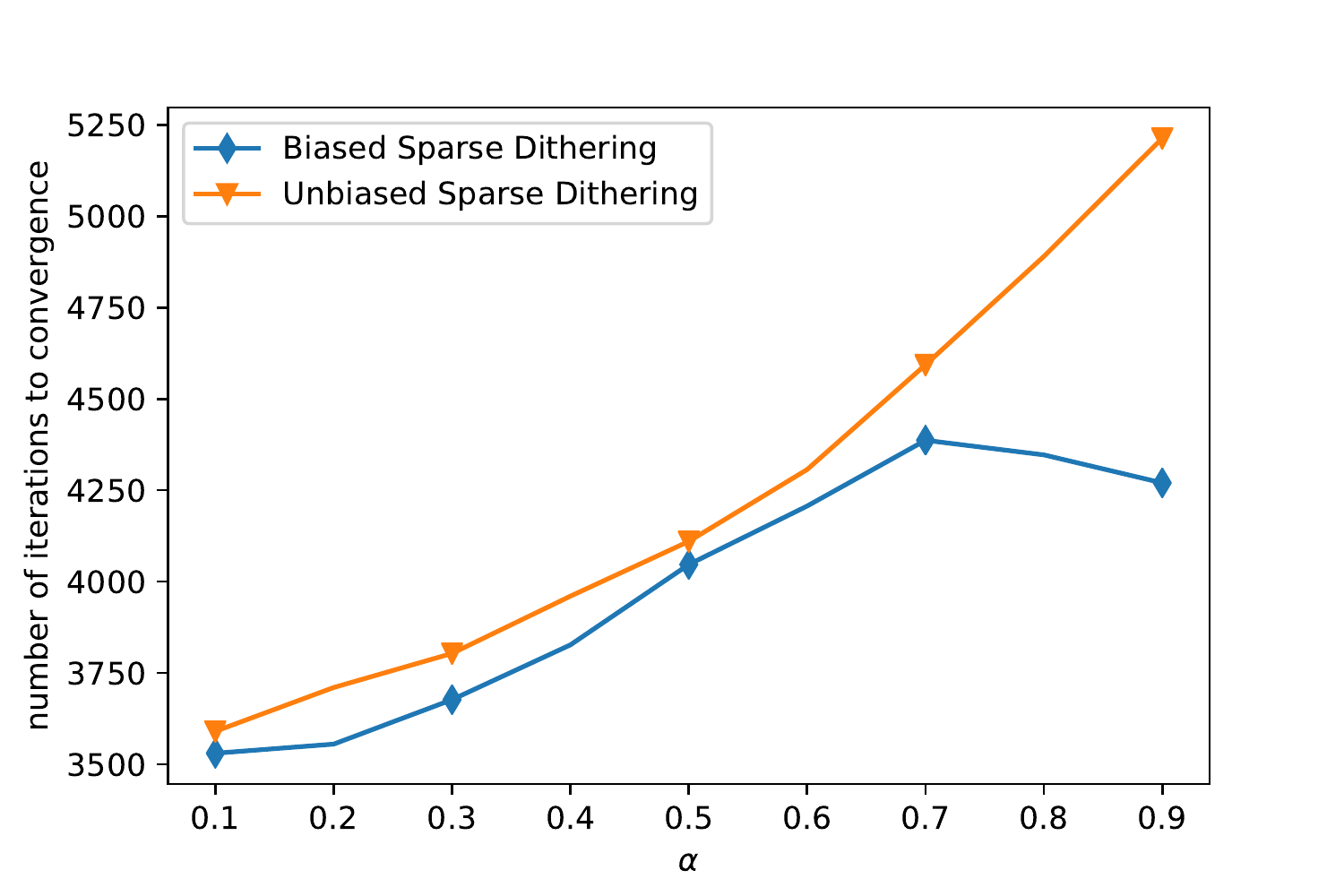}
    \includegraphics[scale=0.53]{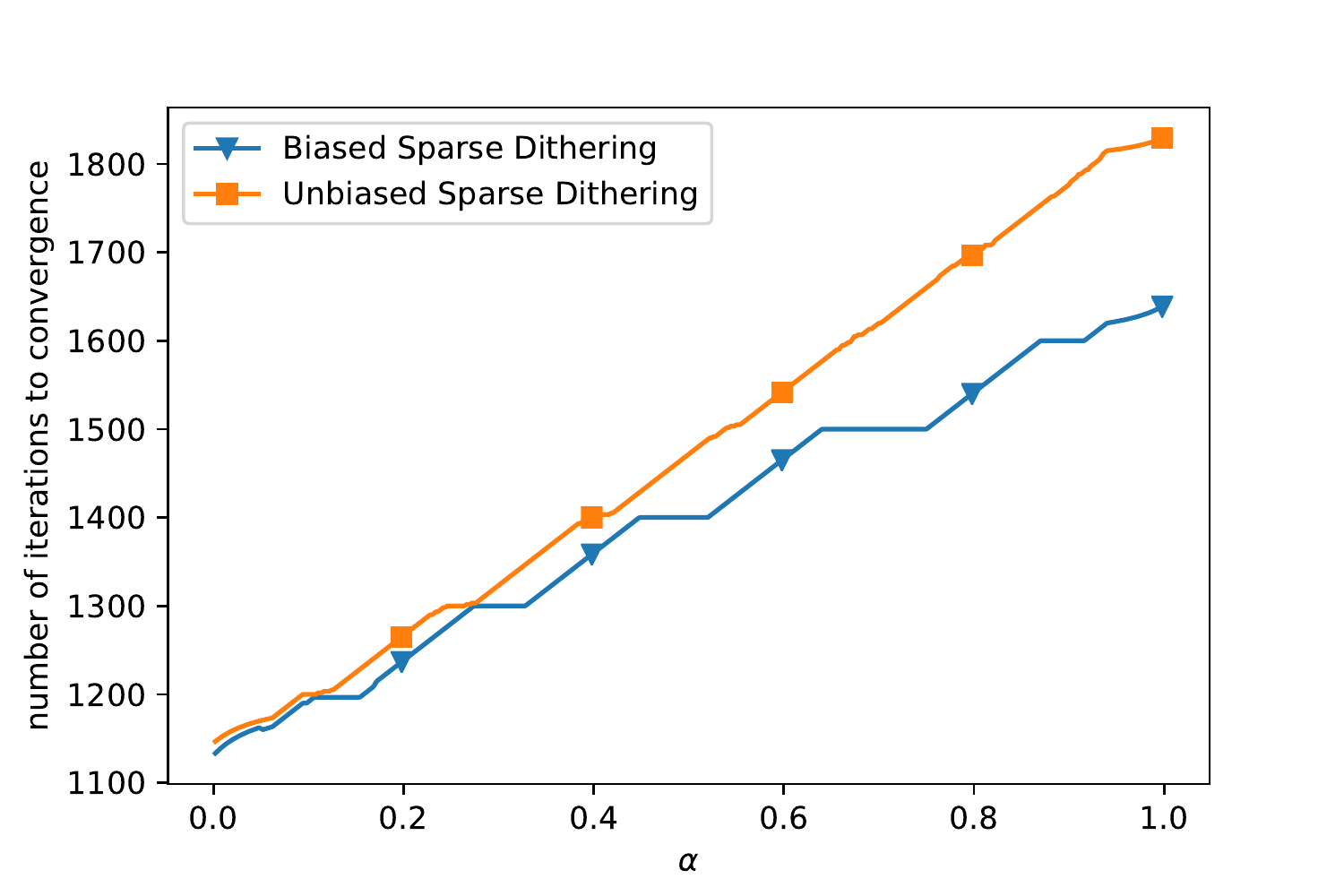}
    \includegraphics[scale=0.53]{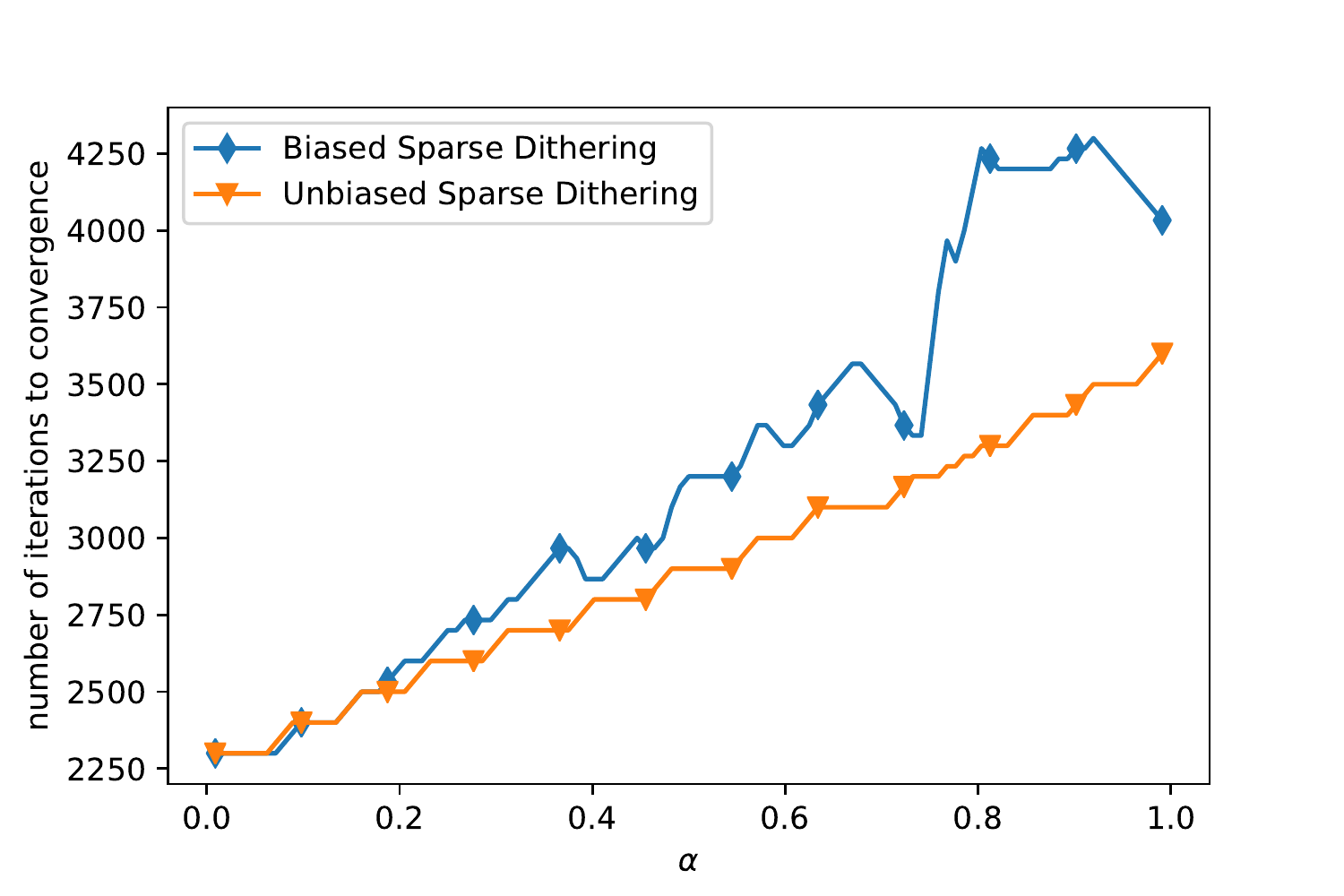}
\end{center} 
\caption{The first two plots correspond to ridge regression (Housing, Bodyfat datasets), while the next three plots correspond to regularized logistic regression (Breast Cancer, Madelon, Mushrooms datasets).   This shows the number of iterations as a function of $\alpha$ for both sparse dithering methods. Both curves look like $Y=1+X$, as predicted by the theory.}
\end{figure}

\begin{figure}[t!]
\begin{center}
    \includegraphics[scale=0.53]{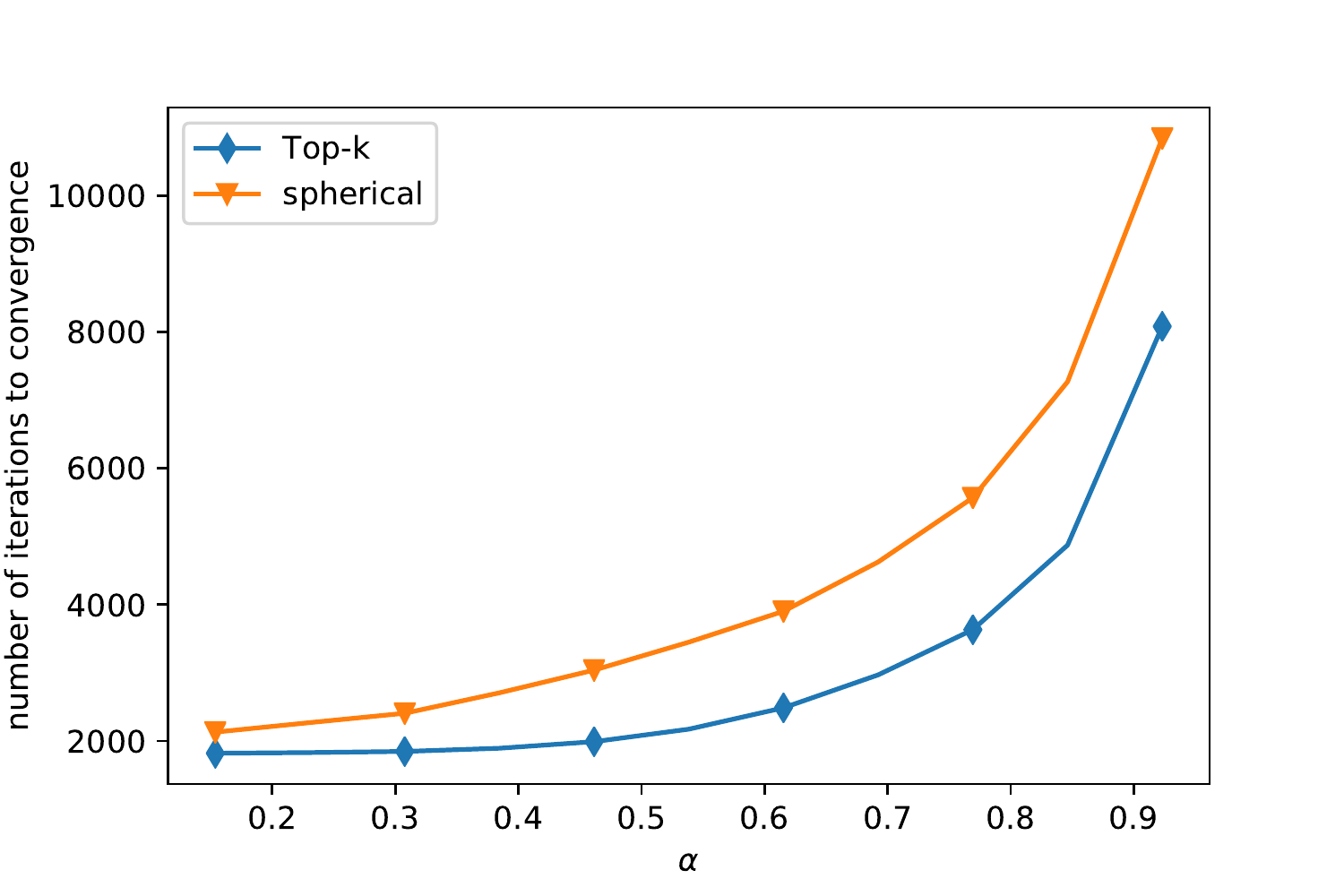}
    \includegraphics[scale=0.53]{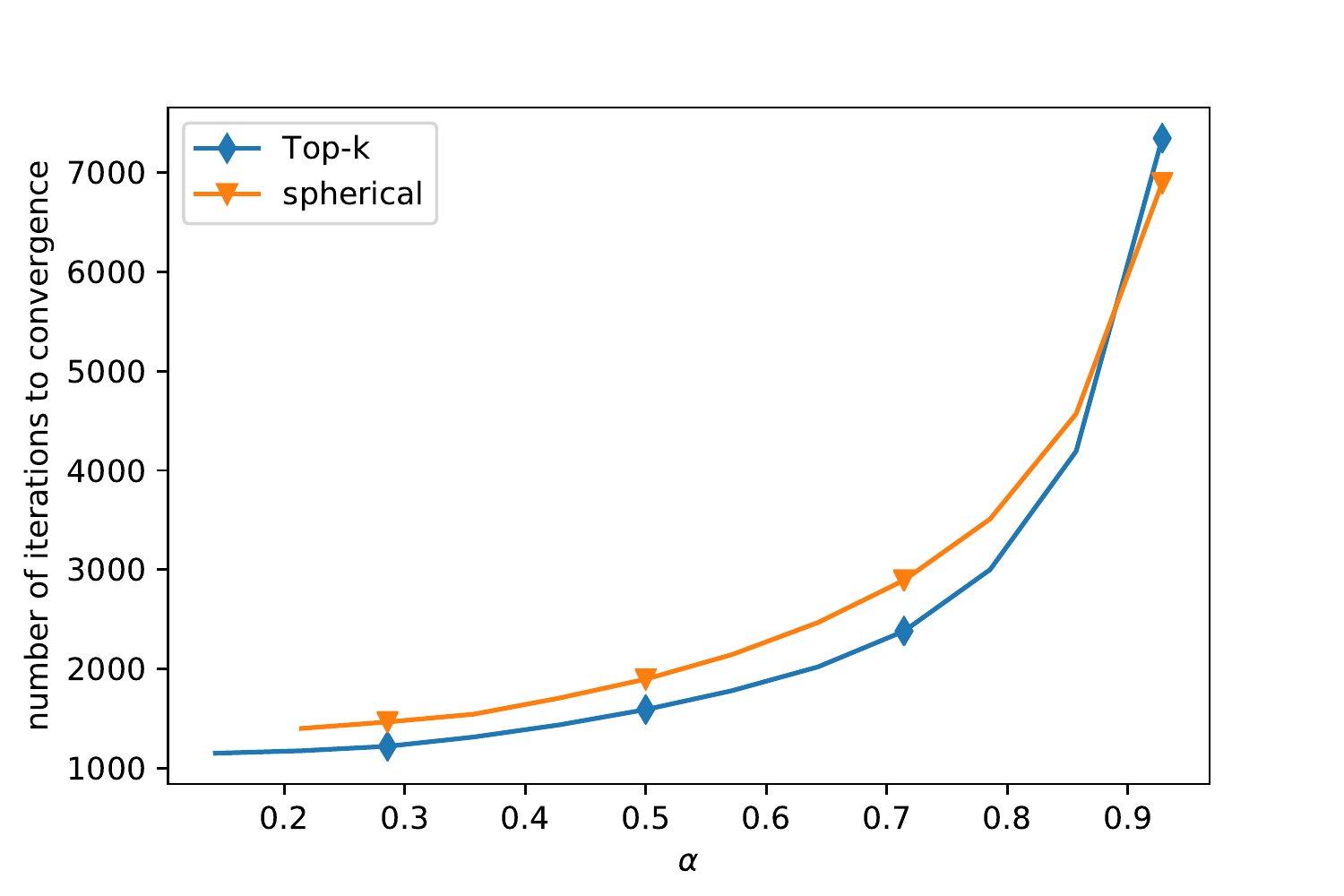}
    \includegraphics[scale=0.53]{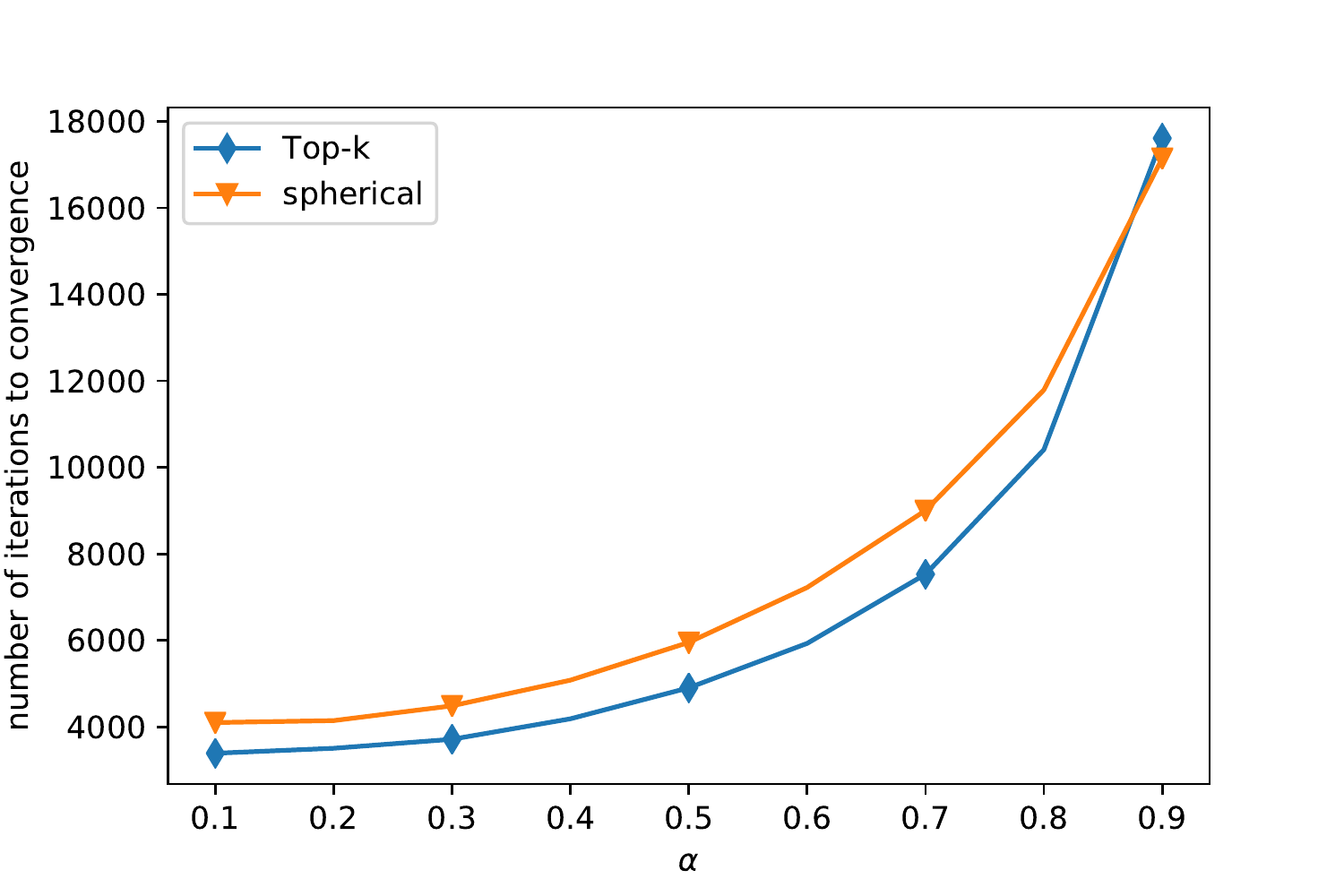}
    \includegraphics[scale=0.53]{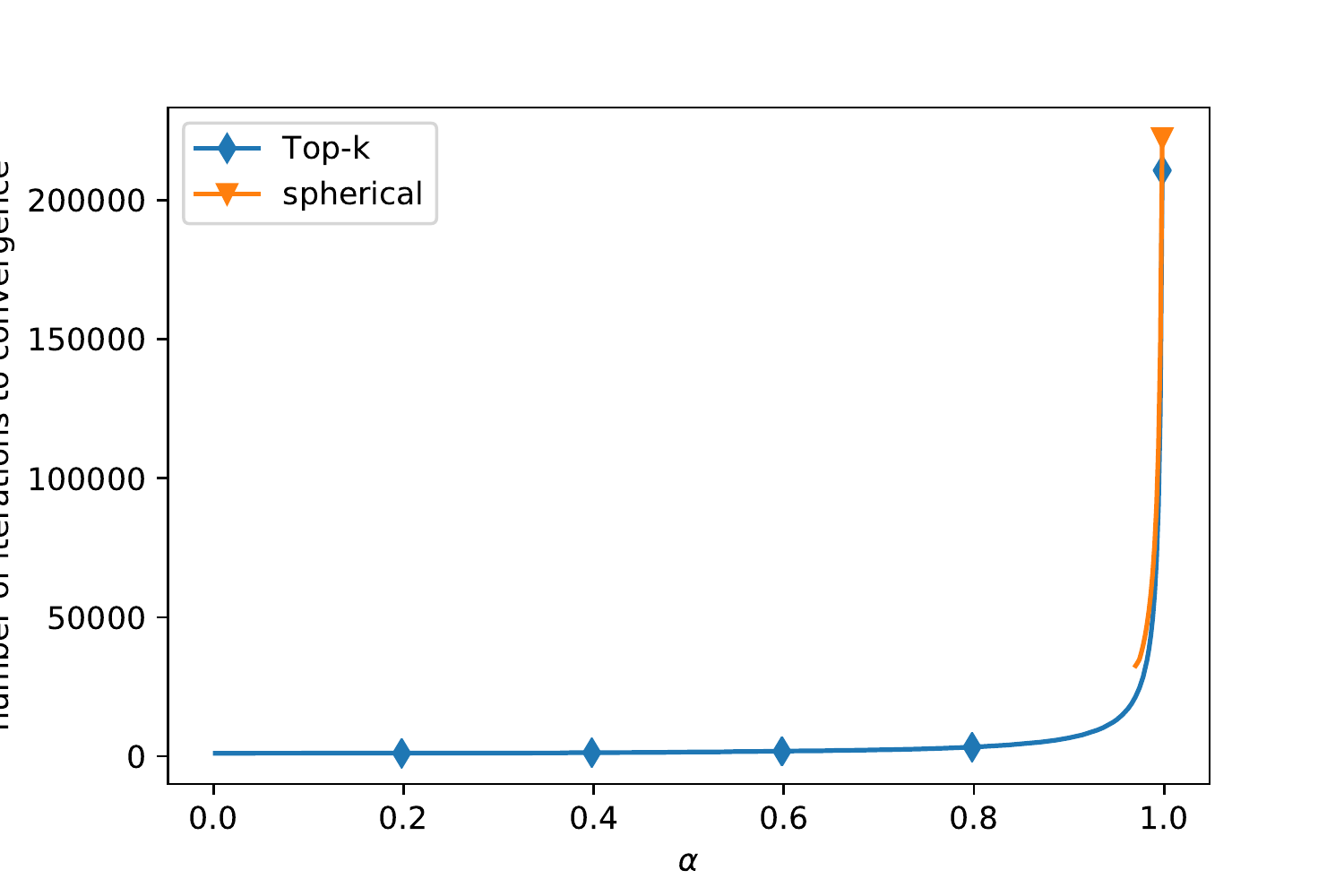}
    \includegraphics[scale=0.53]{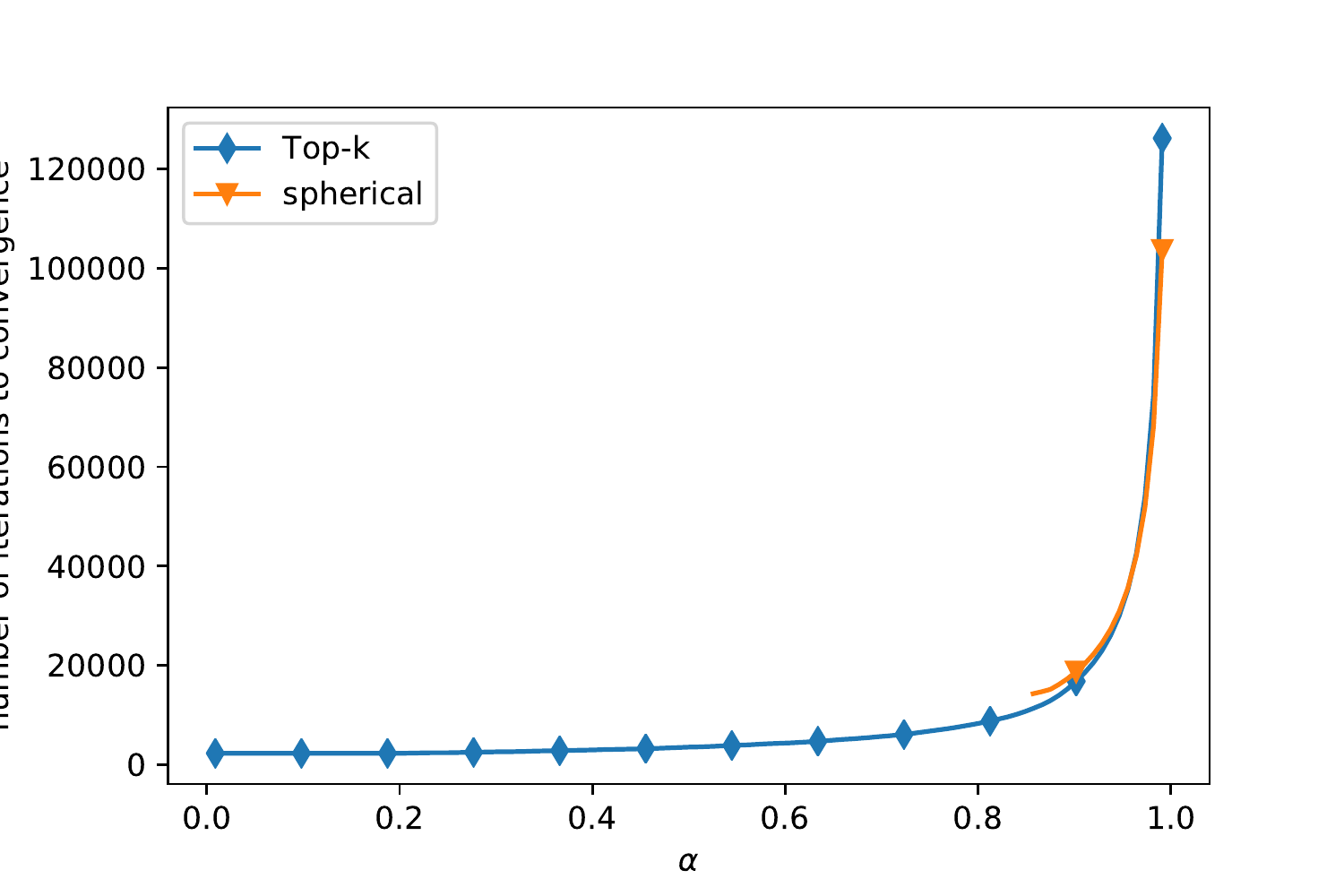}
\end{center} 
\caption{The first two plots correspond to ridge regression (Housing, Bodyfat datasets), while the next three plots correspond to regularized logistic regression (Breast Cancer, Madelon, Mushrooms datasets). This shows the number of iterations as a function of $\alpha$ for both Top-$k$ and Spherical compressions. Both curves look like $Y=\frac{1}{1-X}$, as predicted by the theory.}
\end{figure}


\clearpage
\bibliographystyle{plainnat}
\bibliography{references}

\begin{thebibliography}{42}
\providecommand{\natexlab}[1]{#1}
\providecommand{\url}[1]{\texttt{#1}}
\expandafter\ifx\csname urlstyle\endcsname\relax
  \providecommand{\doi}[1]{doi: #1}\else
  \providecommand{\doi}{doi: \begingroup \urlstyle{rm}\Url}\fi

\bibitem[Alistarh et~al.(2017{\natexlab{a}})Alistarh, Grubic, Li, Tomioka, and
  Vojnovic]{AGLTV}
Dan Alistarh, Demjan Grubic, Jerry Li, Ryota Tomioka, and Milan Vojnovic.
\newblock {QSGD}: {C}ommunication-efficient {SGD} via gradient quantization and
  encoding.
\newblock In \emph{Neural Information Processing Systems Conf. (NeurIPS)},
  2017{\natexlab{a}}.

\bibitem[Alistarh et~al.(2017{\natexlab{b}})Alistarh, Grubic, Li, Tomioka, and
  Vojnovic]{qsgd}
Dan Alistarh, Demjan Grubic, Jerry Li, Ryota Tomioka, and Milan Vojnovic.
\newblock {QSGD}: Communication-efficient sgd via gradient quantization and
  encoding.
\newblock In \emph{Neural Information Processing Systems Conf. (NeurIPS)},
  2017{\natexlab{b}}.

\bibitem[Alistarh et~al.(2018)Alistarh, Hoefler, Johansson, Khirirat,
  Konstantinov, and Renggli]{Alistarh-SparsGradMethods2018}
Dan Alistarh, Torsten Hoefler, Mikael Johansson, Sarit Khirirat, Nikola
  Konstantinov, and C\'{e}dric Renggli.
\newblock The convergence of sparsified gradient methods.
\newblock In \emph{Neural Information Processing Systems Conf. (NeurIPS)},
  2018.

\bibitem[Basu et~al.(2019)Basu, Data, Karakus, and Diggavi]{basu}
Debraj Basu, Deepesh Data, Can Karakus, and Suhas~N. Diggavi.
\newblock Qsparse-local-{SGD}: Distributed {SGD} with quantization,
  sparsification, and local computations.
\newblock In \emph{Neural Information Processing Systems Conf. (NeurIPS)},
  2019.

\bibitem[Bekkerman et~al.(2011)Bekkerman, Bilenko, and
  Langford]{bekkerman2011scaling}
Ron Bekkerman, Mikhail Bilenko, and John Langford.
\newblock \emph{Scaling up machine learning: Parallel and distributed
  approaches}.
\newblock Cambridge University Press, 2011.

\bibitem[Beznosikov et~al.(2020)Beznosikov, Horv{\'a}th, Richt{\'a}rik, and
  Safaryan]{bez20}
Alexandre Beznosikov, Samuel Horv{\'a}th, Peter Richt{\'a}rik, and Mher
  Safaryan.
\newblock On biased compression for distributed learning.
\newblock preprint arXiv:2002.12410, 2020.

\bibitem[B{\"o}r{\"o}czky and Wintsche(2003)]{BorWin}
K{\'a}roly B{\"o}r{\"o}czky and Gergely Wintsche.
\newblock Covering the sphere by equal spherical balls.
\newblock In Boris Aronov, Saugata Basu, J{\'a}nos Pach, and Micha Sharir,
  editors, \emph{Discrete and Computational Geometry: The Goodman-Pollack
  Festschrift}, pages 235--251. Springer, Berlin, Heidelberg, 2003.

\bibitem[Bu et~al.(2020)Bu, Gao, Zou, and Veeravalli]{YWSV-modelcompression}
Yuheng Bu, Weihao Gao, Shaofeng Zou, and Venugopal~V. Veeravalli.
\newblock Information-theoretic understanding of population risk improvement
  with model compression.
\newblock In \emph{AAAI Conference on Artificial Intelligence}, pages
  3300--3307, February 2020.

\bibitem[Cover and Thomas(2006)]{elements_IT}
Thomas~M. Cover and Joy~A. Thomas.
\newblock \emph{Elements of Information Theory (Wiley Series in
  Telecommunications and Signal Processing)}.
\newblock Wiley-Interscience, USA, 2006.

\bibitem[Dryden et~al.(2016)Dryden, Moon, Jacobs, and Essen]{Dryden2016:topk}
N.~Dryden, T.~Moon, S.~A. Jacobs, and B.~V. Essen.
\newblock Communication quantization for data-parallel training of deep neural
  networks.
\newblock In \emph{2nd Workshop on Machine Learning in HPC Environments
  (MLHPC)}, pages 1--8, Nov 2016.

\bibitem[Dumer(2007)]{Ilya2007}
Ilya Dumer.
\newblock Covering spheres with spheres.
\newblock \emph{Discrete {\&} Computational Geometry}, 38:\penalty0 665--679,
  2007.

\bibitem[Gao et~al.(2019)Gao, Liu, Wang, and Oh]{GLWO-modelcompression}
Weihao Gao, Yu-Han Liu, Chong Wang, and Sewoong Oh.
\newblock Rate distortion for model {C}ompression:{F}rom theory to practice.
\newblock In \emph{Int. Conf. Machine Learning (ICML)}, volume PMLR 97, pages
  2102--2111, 2019.

\bibitem[Horv\'{a}th et~al.(2019{\natexlab{a}})Horv\'{a}th, Ho, Horv\'{a}th,
  Sahu, Canini, and Richt\'{a}rik]{Cnat}
Samuel Horv\'{a}th, Chen-Yu Ho, \v{L}udov\'{i}t Horv\'{a}th, Atal~Narayan Sahu,
  Marco Canini, and Peter Richt\'{a}rik.
\newblock Natural compression for distributed deep learning.
\newblock preprint arXiv:1905.10988, 2019{\natexlab{a}}.

\bibitem[Horv\'{a}th et~al.(2019{\natexlab{b}})Horv\'{a}th, Kovalev,
  Mishchenko, Stich, and Richt\'{a}rik]{DIANA-VR}
Samuel Horv\'{a}th, Dmitry Kovalev, Konstantin Mishchenko, Sebastian Stich, and
  Peter Richt\'{a}rik.
\newblock Stochastic distributed learning with gradient quantization and
  variance reduction.
\newblock preprint arXiv:1904.05115, 2019{\natexlab{b}}.

\bibitem[Karimireddy et~al.(2019{\natexlab{a}})Karimireddy, Kale, Mohri, Reddi,
  Stich, and Suresh]{Karimireddy2019}
Sai~Praneeth Karimireddy, Satyen Kale, Mehryar Mohri, Sashank~J. Reddi,
  Sebastian~U. Stich, and Ananda~Theertha Suresh.
\newblock {SCAFFOLD: S}tochastic controlled averaging for on-device federated
  learning.
\newblock preprint arXiv:1910.06378, 2019{\natexlab{a}}.

\bibitem[Karimireddy et~al.(2019{\natexlab{b}})Karimireddy, Rebjock, Stich, and
  Jaggi]{karimireddy2019error}
Sai~Praneeth Karimireddy, Quentin Rebjock, Sebastian~U Stich, and Martin Jaggi.
\newblock Error feedback fixes {S}ign{SGD} and other gradient compression
  schemes.
\newblock preprint arXiv:1901.09847, 2019{\natexlab{b}}.

\bibitem[Khaled et~al.(2020)Khaled, Mishchenko, and
  Richt\'{a}rik]{localSGD-AISTATS2020}
Ahmed Khaled, Konstantin Mishchenko, and Peter Richt\'{a}rik.
\newblock Tighter theory for local {SGD} on identical and heterogeneous data.
\newblock In \emph{Int. Conf. Artificial Intelligence and Statistics
  (AISTATS)}, 2020.

\bibitem[Khirirat et~al.(2018)Khirirat, Feyzmahdavian, and
  Johansson]{khirirat2018distributed}
Sarit Khirirat, Hamid~Reza Feyzmahdavian, and Mikael Johansson.
\newblock Distributed learning with compressed gradients.
\newblock preprint arXiv:1806.06573, 2018.

\bibitem[Kochol(1994)]{Koch2}
Martin Kochol.
\newblock Constructive approximation of a ball by polytopes.
\newblock \emph{Mathematica Slovaca}, 44\penalty0 (1):\penalty0 99--105, 1994.

\bibitem[Kone\v{c}n\'{y} and Richt\'{a}rik(2018)]{RDME}
Jakub Kone\v{c}n\'{y} and Peter Richt\'{a}rik.
\newblock Randomized distributed mean estimation: accuracy vs communication.
\newblock \emph{Frontiers in Applied Mathematics and Statistics}, 4\penalty0
  (62):\penalty0 1--11, 2018.

\bibitem[Kone\v{c}n\'{y} et~al.(2016)Kone\v{c}n\'{y}, McMahan, Yu,
  Richt\'{a}rik, Suresh, and Bacon]{FEDLEARN}
Jakub Kone\v{c}n\'{y}, H.~Brendan McMahan, Felix Yu, Peter Richt\'{a}rik,
  Ananda~Theertha Suresh, and Dave Bacon.
\newblock Federated learning: strategies for improving communication
  efficiency.
\newblock In \emph{NIPS Private Multi-Party Machine Learning Workshop}, 2016.

\bibitem[Li et~al.(2020)Li, Kovalev, Qian, and Richt{\'a}rik]{AccCGD}
Zhize Li, Dmitry Kovalev, Xun Qian, and Peter Richt{\'a}rik.
\newblock Acceleration for compressed gradient descent in distributed and
  federated optimization.
\newblock In \emph{Int. Conf. Machine Learning (ICML)}, 2020.

\bibitem[Lin et~al.(2018)Lin, Han, Mao, Wang, and Dally]{LHMWD}
Yujun Lin, Song Han, Huizi Mao, Yu~Wang, and William~J. Dally.
\newblock Deep gradient compression: Reducing the communication bandwidth for
  distributed training.
\newblock In \emph{Int. Conf. Learning Representations (ICLR)}, 2018.

\bibitem[McMahan et~al.(2017)McMahan, Moore, Ramage, Hampson, and Ag\"{u}era~y
  Arcas]{FL2017-AISTATS}
H~Brendan McMahan, Eider Moore, Daniel Ramage, Seth Hampson, and Blaise
  Ag\"{u}era~y Arcas.
\newblock Communication-efficient learning of deep networks from decentralized
  data.
\newblock In \emph{Int. Conf. Artificial Intelligence and Statistics
  (AISTATS)}, 2017.

\bibitem[Mishchenko et~al.(2019)Mishchenko, Gorbunov, Tak\'a\v{c}, and
  Richt\'arik]{MGTR}
Konstantin Mishchenko, Eduard Gorbunov, Martin Tak\'a\v{c}, and Peter
  Richt\'arik.
\newblock Distributed learning with compressed gradient differences.
\newblock preprint arXiv:1901.09269, 2019.

\bibitem[Safaryan et~al.(2020)Safaryan, Shulgin, and
  Richt\'{a}rik]{up_kashin_2020}
Mher Safaryan, Egor Shulgin, and Peter Richt\'{a}rik.
\newblock Uncertainty principle for communication compression in distributed
  and federated learning and the search for an optimal compressor.
\newblock preprint arXiv:2002.08958, 2020.

\bibitem[Schmidhuber(2015)]{Sch}
J\"urgen Schmidhuber.
\newblock Deep learning in neural networks: An overview.
\newblock In \emph{Neural networks}, volume~61, page 85–117, 2015.

\bibitem[Seide et~al.(2014)Seide, Fu, Droppo, Li, and Yu]{SFDLY}
Frank Seide, Hao Fu, Jasha Droppo, Gang Li, and Dong Yu.
\newblock 1-bit stochastic gradient descent and application to data-parallel
  distributed training of speech {DNN}s.
\newblock In \emph{Fifteenth Annual Conference of the International Speech
  Communication Association}, 2014.

\bibitem[Shannon(1948)]{shannon-1}
C.E. Shannon.
\newblock Coding theorems for a discrete source with a fidelity criterion.
\newblock \emph{IRE Nat. Conv. Rec.}, 27:\penalty0 379--423,623--656, 1948.

\bibitem[Shannon(1959)]{shannon-2}
C.E. Shannon.
\newblock A mathematical theory of communication.
\newblock \emph{Bell Syst. Tech. J.}, 4:\penalty0 142--163, 1959.

\bibitem[Stich(2019)]{Stich2018:localsgd}
Sebastian~U. Stich.
\newblock Local {SGD} converges fast and communicates little.
\newblock In \emph{Int. Conf. Learning Representations (ICLR)}, 2019.

\bibitem[Stich and Karimireddy(2019)]{stich2019}
Sebastian~U. Stich and Sai~Praneeth Karimireddy.
\newblock The error-feedback framework: Better rates for {SGD} with delayed
  gradients and compressed communication.
\newblock preprint arXiv:1909.05350, 2019.

\bibitem[Suresh et~al.(2017)Suresh, Yu, Kumar, and McMahan]{Suresh2017}
Ananda~Theertha Suresh, Felix~X. Yu, Sanjiv Kumar, and H.~Brendan McMahan.
\newblock Distributed mean estimation with limited communication.
\newblock In \emph{Int. Conf. Machine Learning (ICML)}, 2017.

\bibitem[Tang et~al.(2019)Tang, Yu, Lian, Zhang, and Liu]{DoubleSqueeze2019}
Hanlin Tang, Chen Yu, Xiangru Lian, Tong Zhang, and Ji~Liu.
\newblock $\texttt{DoubleSqueeze}$: Parallel stochastic gradient descent with
  double-pass error-compensated compression.
\newblock In \emph{Int. Conf. Machine Learning}, volume PMLR 97, pages
  6155--6165, 2019.

\bibitem[Vaswani et~al.(2019)Vaswani, Bach, and Schmidt]{Vaswani2019-overparam}
Sharan Vaswani, Francis Bach, and Mark Schmidt.
\newblock Fast and faster convergence of {SGD} for over-parameterized models
  and an accelerated perceptron.
\newblock In \emph{Int. Conf. Artificial Intelligence and Statistics
  (AISTATS)}, 2019.

\bibitem[Vogels et~al.(2019)Vogels, Karimireddy, and Jaggi]{vogels}
Thijs Vogels, Sai~Praneeth Karimireddy, and Martin Jaggi.
\newblock Power{SGD}: Practical low-rank gradient compression for distributed
  optimization.
\newblock In \emph{Neural Information Processing Systems Conf. (NeurIPS)},
  2019.

\bibitem[Wang et~al.(2018)Wang, Sievert, Liu, Charles, Papailiopoulos, and
  Wright]{WSLCPW}
Hongyi Wang, Scott Sievert, Shengchao Liu, Zachary Charles, Dimitris
  Papailiopoulos, and Stephen Wright.
\newblock Atomo: Communication-efficient learning via atomic sparsification.
\newblock In \emph{Neural Information Processing Systems Conf. (NeurIPS)},
  2018.

\bibitem[Wangni et~al.(2018)Wangni, Wang, Liu, and Zhang]{tonko}
Jianqiao Wangni, Jialei Wang, Ji~Liu, and Tong Zhang.
\newblock Gradient sparsification for communication-efficient distributed
  optimization.
\newblock In \emph{Neural Information Processing Systems Conf. (NeurIPS)},
  2018.

\bibitem[Wen et~al.(2017)Wen, Xu, Yan, Wu, Wang, Chen, and Li]{terngrad}
Wei Wen, Cong Xu, Feng Yan, Chunpeng Wu, Yandan Wang, Yiran Chen, and Hai Li.
\newblock Terngrad: Ternary gradients to reduce communication in distributed
  deep learning.
\newblock In \emph{Neural Information Processing Systems Conf. (NeurIPS)}, page
  1509–1519, 2017.

\bibitem[Zhang et~al.(2017)Zhang, Li, Kara, Alistarh, Liu, and Zhang]{ZLKALZ}
Hantian Zhang, Jerry Li, Kaan Kara, Dan Alistarh, Ji~Liu, and Ce~Zhang.
\newblock {ZipML}: Training linear models with end-to-end low precision, and a
  little bit of deep learning.
\newblock In \emph{Int. Conf. Machine Learning (ICML)}, volume~70, page
  4035–4043, 2017.

\bibitem[Zhang et~al.(2013)Zhang, Duchi, Jordan, and Wainwright]{ZDJWM}
Yuchen Zhang, John Duchi, Michael~I Jordan, and Martin~J Wainwright.
\newblock Information-theoretic lower bounds for distributed statistical
  estimation with communication constraints.
\newblock In \emph{Advances in Neural Information Processing Systems 26
  (NIPS)}, pages 2328--2336, 2013.

\bibitem[Zheng et~al.(2019)Zheng, Huang, and Kwok]{zheng}
Shuai Zheng, Ziyue Huang, and James~T. Kwok.
\newblock Communication-efficient distributed blockwise momentum {SGD} with
  error-feedback.
\newblock In \emph{Neural Information Processing Systems Conf. (NeurIPS)},
  2019.

\end{thebibliography}

\newpage

\appendix

\part*{Appendix}

\section{Discussion on Finite Precision Floats}

In the paper, we formally consider compression of arbitrary vectors of $\mathbb{R}^d$, but in practice, in computers, real numbers are represented with finite-precision floats, typically using 32 bits. As a consequence, a nonzero real cannot be too small, too large, and its precision is limited. Compression, like every operation, amounts to a sequence of elementary arithmetic operations, each being exact only up to so-called `machine precision', which is difficult to model. Thus, we could restrict ourselves to vectors in 
$\{0\}\cup \{x : r\leq \|x\|\leq R\}$ for some $0<r<R$, instead of the whole space $\mathbb{R}^d$, but this would not account for the finite precision of floats, and since this set is not stable by arithmetic operations, this would not be enough to model the setting in all rigor. So, we prefer to stick with the general setting of $\mathbb{R}^d$ throughout the paper, since there is no issue with limit cases of  very large or very small nonzero numbers, that would deserve a particular discussion; the finite precision makes them   automatically  irrelevant in practice. In other words, the finite representation of reals is not more problematic with compression than for any learning or optimization task, and more generally for the numerical implementation of any mathematical algorithm.

In particular, considering floats with 32 bits, a non-compressed vector $x$ of $\mathbb{R}^d$ is actually represented using $32d$ bits. When we decompose $x$ into its `gain' $\|x\|$ and `shape' $x/\|x\|\in\mathbb{S}^d$, there is no trickery in considering that $\|x\|$ is represented using 31 bits (the sign bit can be omitted) and that $x/\|x\|$ is actually compressed. The multiplication by $\|x\|$ at decompression has finite precision, just like any arithmetic operation.

\section{Proofs for Section \ref{sec:operator-classes}}

\subsection{Relaxed classes of compression operators}

As mentioned in the paper and in Appendix~A, any operator from $\U(\omega),\,\B(\alpha),\,\C(\alpha)$ cannot be encoded with a finite number of bits. For example, in the case of $\B(\alpha)$, the inequality (\ref{class-biased}) breaks near $x=0$ and $\|x\|\rightarrow \infty$. However, in practice, machine floats have finite precision and we do not deal with nonzero values that are too small and too large. To reflect this practical aspect into the theory, we adjust the definition of $\alpha$-contractive compressors $\B(\alpha)$ and consider the following class instead. For the sake of concreteness, we carry out the discussion for the class $\B(\alpha)$ only and note that analogous observations can be adopted for other two classes.

\begin{definition}[Practical $\alpha$-contractive compressions]
Let $\alpha\in(0,1)$ and $R\ge1$ be fixed. We denote by $\B^1(\alpha, R)$ the class of (possibly randomized) operators $\cC\colon \R^d\to\R^d$ such that
\begin{align*}
\E\[\|\mathcal{C}(x)-x\|^2\] \le \alpha\|x\|^2, \quad &\text{if} \quad \nicefrac{1}{R}\le\|x\|\le R \\
\cC(x) = 0, \quad &\text{if} \quad \|x\| < \nicefrac{1}{R} \;\text{ or }\; \|x\| > R.
\end{align*}
\end{definition}

Note that, for simplicity, we take $r=\nicefrac{1}{R}$.

The class of all $\alpha$-contractive operators $\B(\alpha)$ can be seen as the limit of the class $\B^1(\alpha, R)\to\B(\alpha)$ as $R\to\infty$. 
The advantage of the class $\B^1(\alpha, R)$ compared to $\B(\alpha)$ is that it allows an encoding with finite number of bits. Next, we relax the definition of $\B^1(\alpha, R)$ as follows:

\begin{definition}[Weak $\alpha$-contractive compressions]
Let $\alpha\in(0,1)$ and $R>0$ be fixed. We denote by $\B^2(\alpha, R)$ the class of (possibly randomized) operators $\cC\colon \R^d\to\R^d$ such that
\begin{align*}
\E\[\|\mathcal{C}(x)-x\|^2\] \le \alpha R^2, \quad &\text{if} \quad \|x\| \le R \\
\cC(x) = 0, \quad &\text{if} \quad \|x\| > R.
\end{align*}
\end{definition}

The following simple lemma shows that the latter class is much more general and contains the first class.
\begin{lemma} If $R\ge\alpha^{-\nicefrac{1}{4}}$ then $\B^1(\alpha, R) \subset \B^2(\alpha, R)$. \end{lemma}
\begin{proof}
Let $\cC\in\B^1(\alpha, R)$ with $\alpha R^4 \ge 1$. If $\|x\| < \nicefrac{1}{R}$ then
$$
\E\[\|\mathcal{C}(x)-x\|^2\] = \|x\|^2 < \frac{1}{R^2} \le \alpha R^2.
$$

If $\nicefrac{1}{R}\le\|x\|\le R$ then
$$
\E\[\|\mathcal{C}(x)-x\|^2\] \le \alpha \|x\|^2 \le \alpha R^2.
$$
\end{proof}

Now, the lower bound $\alpha\, 4^{\nicefrac{b}{d}} \ge 1$ was proved for any $\cC\in\B^2(\alpha, R)$ and hence for any $\cC\in\B^1(\alpha, R)$ with sufficiently large $R$. Since this lower bound is independent of $R$, it can be associated with the limit class $\B(\alpha)$ under the described practical caveat.

Lastly, we define another class of contractive compression operators which will be used to provide some examples related to the optimality.

\begin{definition}[Spherical $\alpha$-contractive compressions]
Let $\alpha\in(0,1)$ and $R\ge1$ be fixed. We denote by $\B^3(\alpha, R)$ the class of (possibly randomized) operators $\cC\colon \R^d\to\R^d$ such that
\begin{align*}
\E\[\|\mathcal{C}(x)-x\|^2\] \le \alpha, \quad &\text{if} \quad \|x\| = 1 \\
\cC(x) = \|x\| \, \cC\(\nicefrac{x}{\|x\|}\), \quad &\text{if} \quad \nicefrac{1}{R} \le \|x\| \le R \\
\cC(x) = 0, \quad &\text{if} \quad \|x\| < \nicefrac{1}{R} \;\text{ or }\; \|x\| > R.
\end{align*}
\end{definition}

The advantage of this class is that any operator $\cC\in\B^3(\alpha, R)$ can be uniquely identified by its restriction $\cC\colon\sphere^{d}\to\R^d$ to the unit sphere. To compress a given vector $x\in\R^d$, we compress its projection $\nicefrac{x}{\|x\|}\in\sphere^{d}$ by applying $\cC$ and we send $\cC\(\nicefrac{x}{\|x\|}\)$ together with the norm $\|x\|\in\R$.

Subsequently, we will concentrate on the compression of unit vectors with as few bits as possible.

\begin{lemma} $\B^3(\alpha, R) \subset \B^1(\alpha, R)$. \end{lemma}
\begin{proof}
Let $\cC\in\B^3(\alpha, R)$. If $\nicefrac{1}{R} \le \|x\| \le R$ then
\begin{align*}
\E\[\|\mathcal{C}(x)-x\|^2\]
&= \E\[\|\nicefrac{\mathcal{C}(x)}{\|x\|} - \nicefrac{x}{\|x\|}\|^2\] \, \|x\|^2 \\
&= \E\[\|\mathcal{C}\(\nicefrac{x}{\|x\|}\) - \nicefrac{x}{\|x\|}\|^2\] \, \|x\|^2 \\
&\le \alpha \|x\|^2.
\end{align*}
The other cases are trivial.
\end{proof}

\begin{lemma}[Lemma 1 in \cite{up_kashin_2020}]\label{lem:inclusion}
If $\mathcal{C}\in\U(\omega)$, then $\frac{1}{\omega+1}\mathcal{C}\in\B(\frac{\omega}{\omega+1})$.
\end{lemma}

\subsection{Two senses of optimality for compression: Proof of Proposition \ref{prop:cgd}}

If $\cC\in\B(\alpha)$ with $\alpha\in[0,1)$, then to minimize $L$-smooth and $\mu$-strongly convex function $f$, CGD needs $\cO\(\frac{1}{1-\alpha}\kappa\log\frac{1}{\epsilon}\)$ steps for $\epsilon$-accuracy, where $\kappa=\frac{L}{\mu}$ is the condition number of $f$ (see e.g. Theorem 13 of \cite{bez20}). If we choose to not use compression operator and send uncompressed gradients ($\alpha=0$) then we get iteration complexity of GD $\cO\(\kappa\log\frac{1}{\epsilon}\)$, which is $\frac{1}{1-\alpha}$ times smaller than for CGD. If compression operator is unbiased with variance $\alpha\ge0$, then the iteration complexity becomes $\cO\((1+\alpha)\kappa\log\frac{1}{\epsilon}\)$ (see e.g. \cite{khirirat2018distributed}). Alternatively, for an unbiased compression operator $\cC\in\U(\alpha)$ one has $\frac{1}{1+\alpha}\cC\in\B\(\frac{\alpha}{1+\alpha}\)$, which implies the iteration complexity $\cO\(\frac{1}{1-\frac{\alpha}{1+\alpha}}\kappa\log\frac{1}{\epsilon}\) = \cO\((1+\alpha)\kappa\log\frac{1}{\epsilon}\)$.

\subsection{Dimension-tolerant compression schemes: Proof of Theorem \ref{thm:dim-tol-op}}

Statement (i) directly follows from (\ref{up-alpha}) and Lemma \ref{lem:inclusion}, since $b \ge d \, \log_4 \frac{1}{\alpha}$ in the biased case and $b \ge d \, \log_4 \frac{\omega}{\omega+1}$ in the unbiased case.

Statement (ii): first we construct an unbiased compression operator on the unit sphere, which together with $\|x\|$ factor will prove the unbiased case. It follows from \cite{Koch2} (see also \cite{up_kashin_2020}, Section 3) that one can construct an unbiased compression operator $\cC\colon\sphere^{d}\to\R^d$ with $\omega = \cO\(\frac{d}{\log\nicefrac{m}{d}}\)$ variance and $\log m$ bits where the dependence of $m$ from $d$ can be up to exponential. Choosing $m=2^{c d - 31}$, we obtain a number of $c d$  bits to encode $\cC(x/\|x\|)$ together with $\|x\|$ and variance
$$
\omega = \cO\(\frac{d}{\log m - \log d}\) = \cO\(\frac{d}{c d - \log d - 31}\) = \cO(\nicefrac{1}{c}).
$$

For the biased case, Lemma \ref{lem:inclusion} implies that the operator $\frac{1}{\omega+1}\cC$ has variance
$$
\alpha = 1 - \frac{1}{\omega+1} = 1 - \frac{1}{\cO(\nicefrac{1}{c})+1} = \frac{1}{1+\Omega(c)}
$$
and uses the same number $c d$ of bits as $\cC$.

\section{Proofs for Section \ref{sec:wca}}

\subsection{Asymptotic tightness of the lower bound (\ref{up-alpha}): Proof of Theorem \ref{thm:tight-construction-alpha}}

First of all, note that to construct a $\alpha$-contractive compression operator $\cC\colon\sphere^{d}\to\R^d$ on the unit sphere, it is sufficient to cover the unit sphere $\sphere^{d}$ by spherical caps generated from balls of radius $\sqrt{\alpha}$. To see this, let $B^d(x^0,\sqrt{\alpha})$ be the ball of radius $\sqrt{\alpha}$ and center $x^0\in\R^d$ and $C^d(x^0,\sqrt{\alpha}) \eqdef B^d(x^0,\sqrt{\alpha})\cap\sphere^{d}$ be the corresponding spherical cap. Then compressing all points $x\in C^d(x^0,\sqrt{\alpha})$ to the center $x^0$ (i.e. $\cC(x)=x^0$) we preserve $\alpha$-contractive property $\|\cC(x) - x\|^2 \le \alpha$ since $\|\cC(x)-x\| = \|x^0-x\| \le \sqrt{\alpha}$.

It can be shown that in order to maximize the surface area of $C^d(x^0,\sqrt{\alpha})$, the center $x^0$ should be on the sphere of radius $\sqrt{1-\alpha}$, namely $\|x^0\| = \sqrt{1-\alpha}$. Based on the formula\footnote{see \url{https://en.wikipedia.org/wiki/Spherical_cap\#Hyperspherical_cap}} for the surface area of spherical caps, we compute the normalized surface area of $C^d(x^0,\sqrt{\alpha})$ to be $\frac{1}{2}I_{\alpha}(\frac{d-1}{2}, \frac{1}{2})$. Thus, $C^d(x^0,\sqrt{\alpha})$ covers $\frac{1}{2}I_{\alpha}(\frac{d-1}{2}, \frac{1}{2})$ portion of the unit sphere $\sphere^{d}$, where $I$ is the regularized incomplete beta function
\begin{equation}\label{def:ri-beta}
I_p(a,b) = \frac{B(p;a,b)}{B(a,b)} = \frac{\int_0^p t^{a-1}(1-t)^{b-1}\,dt}{\int_0^1 t^{a-1}(1-t)^{b-1}\,dt}, \quad a,b > 0,\, p\in[0,1].
\end{equation}

Next, we use the following result on covering the sphere with balls:

\begin{theorem}[see Theorem 1 in \cite{BorWin}]
For any $d\ge 3$ and $r\in(0,1)$, the unit sphere $\sphere^{d}$ can be covered with balls of radius $r$ in a way that no point of $\;\sphere^{d}$ is covered more than $400\, d\ln d$ times.
\end{theorem}

Let $m$ be the number of balls of radius $\sqrt{\alpha}$ that cover the whole unit sphere with density at most $400 \,d\ln d$. This implies that
$$
m \, \frac{1}{2}I_{\alpha}\(\frac{d-1}{2}, \frac{1}{2}\) \le 400\, d\ln d.
$$

Now these $m$ balls can be encoded using $b=\lceil\log m\rceil$ bits, so that $m\ge 2^{b-1}$. Therefore
\begin{equation}\label{temp-lower-bound}
2^b \, I_{\alpha}\(\frac{d-1}{2}, \frac{1}{2}\) \le 1600\, d\ln d.
\end{equation}

It remains to lower bound the function $I$, which we do as follows
\begin{align*}
I_{\alpha}\(\frac{d-1}{2}, \frac{1}{2}\) 
&= \frac{1}{B\(\frac{d-1}{2},\frac{1}{2}\)} \int_0^{\alpha} t^{\frac{d-3}{2}}(1-t)^{-\frac{1}{2}}\;dt = \frac{\Gamma\(\frac{d}{2}\)}{\Gamma\(\frac{d-1}{2}\)\Gamma\(\frac{1}{2}\)} \int_0^{\alpha} t^{\frac{d-3}{2}}(1-t)^{-\frac{1}{2}}\;dt \\
&\ge \frac{1}{\sqrt{\pi}} \int_0^{\alpha} t^{\frac{d-3}{2}}\;dt = \frac{2}{\sqrt{\pi}(d-1)} \alpha^{\frac{d-1}{2}} \ge \frac{\alpha^{\nicefrac{d}{2}}}{d}.
\end{align*}

Applying this lower bound to (\ref{temp-lower-bound}) and making some simplifications we get
\begin{equation}
2^b \, \alpha^{\nicefrac{d}{2}} \le 1600\, d^2\ln d,
\end{equation}
which is equivalent to (\ref{tight-up-bound}). Finally, since $1\le\sqrt[d]{d}\to1$ as $d\to\infty$, we can make the right hand side of (\ref{tight-up-bound}) smaller than $1+\epsilon$ for any fixed $\epsilon$ just by choosing a large $d$.

\subsection{Deterministic-biased version of SD: Proof of Theorem \ref{thm:DSD}}

{\bf Compression operator and variance bound.} To compress a given nonzero vector $x\in\R^d$, we first compress the normalized vector $u = x/\|x\|\in\sphere^{d}$ and then rescale it. To quantize the coordinates of unit vector $u$, we apply dithering with levels $2k_i h,\; k_i\ge 0$, where $h = \sqrt{\nicefrac{\nu}{d}}$ is the half-step and $\nu\ge 0$ is a free parameter of the compression operator. For each $u_i, i\in[d]$ we choose the nearest level so that $||u_i| - 2k_ih|\le h$. Letting $\hat{u}_i = \sign(u_i) \, 2k_ih$ we have $|u_i-\hat{u}_i|\le h$ for all $i\in[d]$. Therefore
$$
\|u - \hat{u}\|^2 = \sum_{i=1}^d (u_i-\hat{u}_i)^2 \le d h^2 = \nu.
$$

Note that, after rescaling with $\|x\|$, this gives a compression with variance at most $\nu$. However, $\|x\|$ is not always the best option. Specifically, we can choose the scaling factor $\gamma>0$ so as to minimize the variance $\|x - \gamma\hat{u}\|^2$, which yields the optimal factor $\gamma^* = \frac{\<x, \hat{u}\>}{\|\hat{u}\|^2}$ with optimal variance $\|x-\gamma^*\hat{u}\|^2 = \sin^2\varphi\,\|x\|^2$, where $\varphi\in[0,\nicefrac{\pi}{2}]$ is the angle between $x$ and $\hat{u}$. Hence, defining the compression operator as $\cC(x) = \gamma^*\hat{u}$, we have the following bound on the variance:
$$
\|\cC(x) - x\|^2 \le \min\(\nu, \sin^2\varphi\) \|x\|^2,
$$
where $\varphi\in[0,\nicefrac{\pi}{2}]$ is the angle between $x$ and $\cC(x)$, and in the case $\cC(x)=0$,  we let $\varphi=\nicefrac{\pi}{2}$.

{\bf Encoding.} Now, we describe  the encoding scheme itself; that is, how many and which bits we need to communicate for $\gamma^*\hat{u}$. We introduce the following notations:
$$
\gamma \eqdef 2h\gamma^*\in\R_+, \quad s \eqdef \(\sign(u_ik_i)\)_{i=1}^d\in\{-1,0,1\}^d, \quad k \eqdef (k_i)_{i=1}^d\in\N_+^d.
$$

Note that $\cC(x) = \gamma^*\hat{u} = 2h\gamma^* \, \sign(u) \, k = \gamma \, s \, k$. So, we need to encode the triple $(\gamma, s, k)$. Since $\gamma\in\R_+$, we need only \fbox{$31$} bits for the scaling factor. Next we encode $s$. Let
$$
n_0 \eqdef \#\{i\in[d] \colon s_i=0\} = \#\{i\in[d] \colon k_i=0\}
$$
be the number of coordinates $u_i$ that are compressed to 0.
To communicate $s$, we first send the locations of those $n_0$ coordinates and then \fbox{$d-n_0$} bits for the values $\pm 1$. Sending $n_0$ positions can be done by sending \fbox{$\log d$} bits\footnote{We can further optimize this with Elias-$\omega$ encoding by sending $\approx\log n_0$ bits instead of $\log d$. However, both are negligible in the overall encoding and we will not complicate the analysis for this small improvement.} representing the number $n_0$, afterwards sending \fbox{$\log\binom{d}{n_0}$} bits for the positions.
Finally, it remains to encode $k$ for which we only need to send nonzero entries since the positions of $k_i=0$ are already encoded. We encode $k_i\ge 1$ with $k_i$ bits: $k_i-1$ ones followed by 0. Hence, encoding $k$ required \fbox{$\sum k_i$} additional bits.

Thus, our encoding scheme for $\cC(x) = \gamma \, s \, k$ is as follows
\begin{itemize}
\item scaling factor $\gamma$: $31$ bits,
\item signs $s$: $\log d + \log\binom{d}{n_0} + d - n_0$ bits,
\item dithering levels $k$: $\sum_{i=1}^d k_i$ bits,
\item total number of bits $b = 31 + \log d + \log\binom{d}{n_0} + d - n_0 + \sum_{i=1}^d k_i$.
\end{itemize}

{\bf Upper bound on $\boldsymbol{b}$.} We continue by giving a theoretical upper bound for the bits $b$ needed to communicate $\cC(x)$. 
Below, we derive an upper bound for $\sum k_i$. Since each $|u_i|$ is quantized to the nearest $2k_ih$, then we have this double bound $\max(0, (2k_i-1)h)\le |u_i| \le (2k_i+1)h$. Using this with the Cauchy--Schwarz inequality we get
$$
1 = \(\sum_{i=1}^d u_i^2\)^{\nicefrac{1}{2}} \ge h \(\sum_{k_i\ge 1} (2k_i-1)^2\)^{\nicefrac{1}{2}} \ge \frac{h}{\sqrt{d-n_0}} \sum_{k_i\ge 1} (2k_i-1),
$$
which implies the following bound on $\sum k_i$:
$$
\sum_{i=1}^d k_i = \sum_{k_i\ge 1} k_i \le \frac{1}{2}\(\frac{\sqrt{d-n_0}}{h} + d-n_0\) = \frac{d}{2} \( \sqrt{\frac{ 1-\nicefrac{n_0}{d}}{\nu}} + 1-\nicefrac{n_0}{d}\).
$$
Setting $\tau = \nicefrac{n_0}{d}\in[0,1]$, we further upper bound it using the AM-GM inequality
$$
\sum_{i=1}^d k_i \le \frac{d}{2}\(\frac{1-\nicefrac{\tau}{2}}{\sqrt{\nu}} + 1-\tau\).
$$
Let us consider the extreme cases $n_0=0$ and $n_0=d$ separately. If $n_0=d$, then $b=31+\log d$. If $n_0=0$, then $b \le 31+ \log d + \(\frac{3}{2}+\frac{1}{2\sqrt{\nu}}\)d$. Note that these extreme cases are the best cases in terms of the number of bits. In the sequel, we assume that $1\le n_0\le d-1$ and hence $\tau\in[\nicefrac{1}{d}, 1-\nicefrac{1}{d}]$.
Next, we upper bound the term $\log\binom{d}{n_0}$, for which it is known the following tight estimate\footnote{``The Theory of Error-Correcting Codes" by MacWilliams and Sloane (Chapter 10, Lemma 7, p. 309)}
$$
\frac{2^{d H_2(\tau)}}{\sqrt{8d\tau(1-\tau)}} \le \binom{d}{\tau d} \le \frac{2^{d H_2(\tau)}}{\sqrt{2\pi d\tau(1-\tau)}}, \quad 0<\tau<1,
$$
where $H_2(\tau) = -\tau\log\tau - (1-\tau)\log(1-\tau)$ is the binary entropy function in bits. Hence
$$
\log\binom{d}{n_0} = \log\binom{d}{\tau d} \le -\frac{1}{2}\log(2\pi d\tau(1-\tau)) + d H_2(\tau).
$$
The first term with negative sign saves at least $\frac{1}{2}\log 2\pi \approx 1.32$ bits and up to $\frac{1}{2}\log\frac{\pi d}{2}$ bits. In further estimations we upper bound it by $-1$. So far, the following upper bound is obtained
\begin{align*}
b
&\le 30 + \log d + d H_2(\tau) + d(1-\tau) + \frac{d}{2}\(\frac{1-\nicefrac{\tau}{2}}{\sqrt{\nu}} + 1-\tau\) \\
&= 30 + \log d + \( H_2(\tau) + \frac{3}{2}(1-\tau) + \frac{1-\nicefrac{\tau}{2}}{2\sqrt{\nu}} \)d \\
&\eqdef 30 + \log d + \beta(\tau, \nu)d.
\end{align*}

It remains to find an upper bound for $\beta(\tau, \nu)$ with respect to $\tau$. As the entropy function $H_2$ and any linear function are concave, we can find the maximum by solving first order optimality condition. The equation $\frac{d}{d\tau}\beta(\tau,\nu)=0$ gives the solution
$$
\tau^* = \frac{1}{ 1+2^{\frac{1}{4}\(6+\frac{1}{\sqrt{\nu}}\)} }.
$$

Setting $\beta(\nu)\eqdef \beta(\tau^*,\nu)$, we upper bound the number of bits $b$ as
$$
b \le 30 + \log d + \beta(\nu)d.
$$

It can be shown that, with $\nu = \nicefrac{1}{10}$, on has $\beta(\nu) \approx 3.3495 < 3.35$. This completes the proof of Theorem \ref{thm:DSD}.

\subsection{Tighter bounds on minimal communication: Proofs of Theorems \ref{thm:logd-bits} and \ref{thm:tighter-bounds}}

The first motivation for this is that even though the uncetainty principle (\ref{up-alpha}) is strong for constant $\alpha$, it is not tight when $\alpha$ goes to 1 as $d$ goes to infinity. In particular, for $\alpha = \frac{d-1}{d}$, it says that the number of bits is at least $d\log(1-1/d)/2$, which is constant. However, we can show that when $\alpha < 1$, one needs at least $\log(d)$ bits. This explains why there is no way to only communicate a fixed number of bits per round while still having $\alpha < 1$. Moreover, we will compute an explicit estimate of $b^*(\alpha, d)$, as a function of $d$ and $\alpha$ only, with a very low error of $\frac{1}{2}\log{\log{d}} + C$ for some absolute constant $C$.

\begin{proof}[Proof of Theorem \ref{thm:logd-bits}]
Proving the result is equivalent to proving that the surface of the unit sphere cannot be covered by less than $d$ smaller, identical balls. We can prove this easily by induction. To make the induction step, let us assume, without loss of generality, that one of the smaller balls is centered on the positive $x_1$ axis. Since the radius of this smaller ball is less than 1, the unit $(d-1)$-dimensional sphere with $x_1=0$ is disjoint from the first smaller ball, which means, by induction, that it will itself require at least $d-1$ additional smaller balls, leading to the desired result. \end{proof}

In fact, the previous result can be used to obtain the following result.

\begin{definition}
For a covering of the surface of the unit sphere using identical spherical caps, we define the density of the cover to be the average number of caps covering a point on the surface of the unit sphere. Identically, this is equal to the number of spherical caps used multiplied by the fraction of the unit sphere covered by a single spherical cap. 
\end{definition}

\begin{theorem}\label{th11}
There exists an absolute constant $B$ such that if the surface of the unit sphere is covered with identical smaller spherical caps, then the density of the covering is at least $B d$.
\end{theorem}

\begin{proof} We split this into two cases. 
The first case is when the radius of the spherical cap is larger than $\sqrt{1-\frac{1}{d}}$. In this case, each spherical cover will cover at least a fraction of $A(d) = \Pr\(x_1 \geq \frac{1}{\sqrt{d}}\)$ where $x$ is chosen uniformly from the surface of the unit sphere. One can easily show that there exists $C_1 > 0$ such that $A(d) > C_1$ for all d. Indeed, it is enough to see that $A(d) > 0$ for all $d$ and that $A(d)$ approaches $1-\phi(1)$ where $\phi$ is the CDF of a standard normal. Combining this bound with the previous result of requiring at least $d$ caps to cover the surface of the unit sphere, the density is at least $C_1d$ when the radius of the spherical cap is at least  $\sqrt{1-\frac{1}{d}}$. \\
In the second case, when the radius of the spherical cap is less than $\sqrt{1-\frac{1}{d}} < \sqrt{1-\frac{1}{d+1}}$, The Coxeter--Few--Rogers ``simplex'' bound shows \cite{Ilya2007} that the density is at least $C_2d$ for some absolute constant $C_2$. Choosing $B=\min\(C_1, C_2\)$ gives us the desired result. \end{proof}

\begin{theorem}[see \cite{Ilya2007}]\label{th12}
There exists an absolute constant $A$ such that for any $d$ and any spherical radius $r < 1$, there exists a cover for the surface of the unit sphere with smaller, identical spherical caps of radius $r$ such that the density of the covering is at most $Ad\log{d}$. 
\end{theorem}

\begin{lemma}\label{lem:max-cap-center}
If $\|x\| = 1$ and $\|v-x\|^2 \le \alpha$ for some $t$, then for $v'=\frac{\sqrt{1-\alpha}v}{\|v\|}$, one has $\|x-v'\|^2 \le \alpha$. In other words, if some balls of radius $\sqrt{\alpha}$ cover the surface of the unit sphere, then projecting them onto the sphere of radius $\sqrt{1-\alpha}$ will still cover entirely the surface of the unit sphere.
\end{lemma}

\begin{proof}
The initial condition is equivalent to $1-\alpha + \|v\|^2 \le 2\langle v,x\rangle$  which, using AM-GM, implies that $2\sqrt{1-\alpha}\|v\| \le 2\langle v,x \rangle$, which can be rearranged to to look like $1-\alpha + \|v'\|^2 \le 2\langle v',x\rangle$ or $\|v'-x\|^2 \le \alpha$, as desired. \end{proof}

The above discussion leads us to our next result on $b^*(\alpha, d)$, which is an important quantity to study. First, 
it tells us that operators $\cC \in \B(\alpha)$ cannot  compressed with less than $b^*$ bits. Moreover, it tells us that this bound is tight, because there is at least one operator in $\B(\alpha)$ that can be compressed to no more than $b^*(\alpha, d)$ bits. Thus, we proceed with estimating $b^*(\alpha, d)$ explicitly with a very small estimation error of $\frac{1}{2}\log{\log{d}} + \cO(1)$.

\begin{proof}[Proof of Theorem \ref{thm:tighter-bounds}] Recall that $P(\alpha, d)=\frac{1}{2}I_{\alpha}(\frac{d-1}{2}, \frac{1}{2})$ is also equal to the fraction of the surface area of the surface of the unit sphere with $\sqrt{1-\alpha} \le x_1$. This can be viewed as the probability that a point $x$ chosen uniformly on the unit sphere satisfies $\sqrt{1-\alpha} \le x$. In order to prove the theorem, we will prove the upper and lower bounds on $b^*(\alpha, d)$ separately. 

We first prove the lower bound. Let $\cC$ be an arbitrary operator in $\B(\alpha)$ that can be encoded with no more than $b$ bits. This means that at most $2^b$ possible values can be communicated. Let $c_1,\dots,c_V$ be all the possible decodings, with $V \le 2^b$. Now, if we consider the balls $B(c_i, \sqrt{\alpha})$, the surface of the unit sphere must be covered. Indeed, if $\|x\|=1$ and $x$ is not covered, then all the possible encodings of $\cC(x)$ will produce a point whose distance from $x$ is more than $\sqrt{\alpha}$, which contradicts the fact that the operator is in $\B(\alpha)$. 

Now, since these small balls cover the surface of the unit sphere, one can use Lemma \ref{lem:max-cap-center} to show that the balls centered at $B(C_i, \sqrt{\alpha})$, where $C_i = \frac{\sqrt{1-\alpha}c_i}{\|c_i\|}$, should also be a covering. 
Using Theorem \ref{th11}, we know that the density of this new coverage is at least $B d$, while it is at most $F 2^b$, where $F$ is the fraction of the surface area of the unit sphere that each one of these balls cover. In fact, one can compute $F$ explicitly as $P(\alpha, d)$ = $\Pr\(x_1 \ge \sqrt{1-\alpha}\)$ where $x$ is chosen uniformly on the surface of the unit sphere. This gives us the lower bound $b \ge -\log{P(\alpha, d)} + \log{d} + \log{B}$.

For the upper bound, one can use a constant number of bits to communicate $\|x\|$, then one can use the covering from Theorem \ref{th12} with radius equal to $\sqrt{\alpha}$ and quantize $x$ to the nearest spherical cap center, which is guarantee to be within a distance of $\sqrt{\alpha}$, ensuring that this quantization is in $\B(\alpha)$. Now, since this covering has density no more than 
$Ad\log{d}$, and since its density is equal to $V P(\alpha, d)$, where $v$ is the number of spherical caps used and $P$ is, as defined above, the fraction of the surface area covered by one a spherical cap of radius $\sqrt{\alpha}$, one can conclude that $v \le \frac{Ad\log{d}}{P(\alpha, d)}$, which means that the centers can be encoded using no more than $-\log{P(\alpha, d)} \log{d} + \log{\log{d}} + \log{A}$ bits, yielding the desired upper bound.\end{proof}

\section{Proofs for Section \ref{sec:aca}}

\subsection{Lower bound on average communication: Proof of Theorem \ref{thm:tighter-bounds-2}}

Let $X$ be a random vector with uniform distribution over the unit sphere $\sphere^d$ and $\hat{X}=\cC(X)$ be the compressed (random) vector. Note that, $\hat{X}$ has two source of randomness, one from the random vector $X$ and the other  coming from the compression operator $\cC$.
Based on the assumption of finiteness of $B$ (otherwise the lower bound is trivial), we conclude that $\cC$, and hence the random vector $\hat{X}$, are discrete; that is, the set of possible values they can take is finite or countably infinite.  
Note that $\hat{X}$ can be encoded with $B$ bits in expectation with respect to its own source of randomness, as
$$
B = \sup_{\|x\|=1}\E_{\cC}\[|E(x)|\] \ge \E_{\cC,X}\[|E(X)|\] = \E_{\hat{X}}\[|E(X)|\].
$$
Thus, the discrete random source $\hat{X}$ admits an encoding with expected binary description length of $B$.
Applying Shanon's source coding theorem on lossless compression\footnote{see e.g. Theorem 5.5.1+Corollary or Theorem 5.11.1 of \cite{elements_IT}}, we get $B\ge H(\hat{X})$, where $H$ indicates the entropy of the source\footnote{A discrete random vector can be mapped to a discrete random variable preserving the same probability distribution (and so we can extend the source-coding inequality), as entropy is defined through probability mass/density function, not the actual values of the random source.} in bits.

Next, using the mutual information and relative entropy of $\hat{X}$ and $X$, we further lower bound it as follows:
$$
B\ge H(\hat{X}) \ge H(\hat{X}) - H(\hat{X}|X) = I(\hat{X}, X) = H(X) - H(X|\hat{X}).
$$
Now, we denote by $A$ the surface area of the unit sphere $\sphere^d$. For a given point $v\in\R^d$, let $A'(v)$ be the surface area of the cap $C^d(v,\sqrt{\alpha}) = B^d(v,\sqrt{\alpha})\cap\sphere^d$. Using Lemma \ref{lem:max-cap-center}, it can be shown that in order to maximize the surface area of $C^d(v,\sqrt{\alpha})$, the center $v$ should be on the sphere of radius $\sqrt{1-\alpha}$, namely $\|v\|_2 = \sqrt{1-\alpha}$. Using the formula\footnote{see \url{https://en.wikipedia.org/wiki/Spherical_cap\#Hyperspherical_cap}} for the surface area of spherical caps, we compute the normalized surface area of $C^d(v,\sqrt{\alpha})$ to be $P(\alpha, d) = \frac{1}{2}I_{\alpha}(\frac{d-1}{2}, \frac{1}{2})$. Thus, at best $C^d(v,\sqrt{\alpha})$ covers the portion $P(\alpha,d)$ of the unit sphere $\sphere^d$, where $I_{\alpha}$ is the regularized incomplete beta function. Therefore, for an arbitrary $v\in\R^d$, one has the upper bound $A'(v)\le P(\alpha,d)A$. Note that as $I_{\alpha}$ is upper bounded by $1$ (which directly follows from the definition) we get $P(\alpha,d)<\nicefrac{1}{2}$.

Since $X$ is uniform on the unit sphere, its probability density function is $\nicefrac{1}{A}$ and so the entropy $H(X) = \log(A)$. Similarly, since the random vector $X$ conditioned with $\hat{X}=v$ is uniform over $C^d(v,\sqrt{\alpha})$, we have $H(X|\hat{X}=v) = \log A'(v)$. Hence 
$$
H(X|\hat{X}) = \E_{\hat{X}}\[H(X|Y=\hat{X})\] = \E_{\hat{X}}\[\log A'(\hat{X})\] \le \log\(P(\alpha, d)A\),
$$
resulting in the desired lower bound
$$
B \ge H(X) - H(X|\hat{X}) \ge \log(A) - \log\(P(\alpha, d)A\) = - \log P(\alpha, d).
$$

\subsection{Randomized-unbiased version of Sparse Dithering: Proof of Theorem \ref{thm:RSD}}

In this section, we randomize Sparse Dithering  to make it unbiased.

{\bf Compression operator and variance bound.} Again, to compress a given nonzero vector $x\in\R^d$, we decompose $x$ into the scalar $\|x\|$ and the unit vector $u=x/\|x\|$. To quantize the coordinates of $u$, we round to one of the two nearest neighbors, so as to preserve unbiasedness; that is, if $2k_ih \le |u_i| \le 2(k_i+1)h$ for some $k_i\ge0$, then
\begin{equation*}
\hat{u}_i = \sign(u_i)\, 2\hat{k}_ih =
\begin{cases}
    \sign(u_i)\, 2k_ih     & \text{ with probability }\ \frac{2(k_i+1)h-|u_i|}{2h}\\ 
    \sign(u_i)\, 2(k_i+1)h & \text{ with probability }\ \frac{|u_i| - 2k_ih}{2h}\\
\end{cases}
\end{equation*}
Clearly, $\E\[\hat{u}\] = u$ and defining $\cC(x) = \|x\|\hat{u}$ we maintain unbiasedness $\E\[\cC(x)\]=x$. Bounding the second moment
\begin{align*}
\E\[\hat{u}_i^2\]
&= \(2k_ih\)^2\frac{2(k_i+1)h-|u_i|}{2h} + \(2(k_i+1)h\)^2 \frac{|u_i| - 2k_ih}{2h} \\
&= u_i^2 + \(|u_i|-2k_ih\)\(2(k_i+1)h-|u_i|\) \\
&\le u_i^2 + \(\frac{|u_i|-2k_ih + 2(k_i+1)h-|u_i|}{2}\)^2 = u_i^2 + h^2,
\end{align*}
we conclude that
$$
\frac{\E\[\|\cC(x)\|^2\]}{\|x\|^2} = \E\[\|\hat{u}\|^2\] \le \sum_{i=1}^d (u_i^2 + h^2) \le 1+dh^2 = 1+\nu.
$$
Hence, the variance of compression operator $\cC$ is $\omega\le\nu$.

{\bf Encoding.} Next, we proceed to the encoding scheme, exactly like in the deterministic case. We introduce the following notations:
$$
\gamma \eqdef 2h\|x\| \in\R_+, \quad s \eqdef \(\sign(u_i\hat{k}_i)\)_{i=1}^d\in\{-1,0,1\}^d, \quad \hat{k} \eqdef (\hat{k}_i)_{i=1}^d\in\N_+^d.
$$

Note that $\cC(x) = \|x\|\hat{u} = 2h\|x\| \, \sign(u) \, \hat{k} = \gamma \, s \, \hat{k}$. So, we need to encode the triple $(\gamma, s, \hat{k})$. The encoding scheme and the formula for the number of bits are the same, so we need to upper bound
$$
\hat{b} = 31 + \log d + \log\binom{d}{\hat{n}_0} + d - \hat{n}_0 + \sum_{i=1}^d \hat{k}_i
$$
in expectation, where $\hat{n}_0 \eqdef \#\{i\in[d] \colon \hat{k}_i=0\}$.

{\bf Upper bound on $\E[\hat{b}]$.} First, notice that
$$
\E\[\sum_{i=1}^d \hat{k}_i\] = \frac{1}{2h}\sum_{i=1}^d \E\[\hat{u}_i\] = \frac{\|u\|_1}{2h} \le \frac{\sqrt{d}}{2h} = \frac{d}{2\sqrt{\nu}}.
$$
Considering $\hat{n}_0=0$ and $\hat{n}_0=d$ cases separately, we get $\E[\hat{b}]\le 31+\log d + \(1+\nicefrac{1}{2\sqrt{\nu}}\)d$ and $\hat{b}=31+\log d$ respectively.
Next, we use the same upper bound for the log-term $\log\binom{d}{\hat{n}_0} \le d H_2(\hat{\tau}) - 1$ with $\hat{\tau} = \nicefrac{\hat{n}_0}{d}\in[\nicefrac{1}{d},1-\nicefrac{1}{d}]$. It remains to upper bound $H_2(\hat{\tau}) + (1-\hat{\tau})$, which is maximized when $\hat{\tau} = \nicefrac{1}{3}$ with value $\log 3$, i.e. $H_2(\hat{\tau}) + (1-\hat{\tau}) \le \log 3$. Thus, we have proved the formula for the number bits in expectation:
$$
\E\[\hat{b}\] \le 31 + \log d + \(d H_2(\hat{\tau}) - 1\) + d(1-\hat{\tau}) + \frac{d}{2\sqrt{\nu}} \le 30 + \log d + \(\log 3 + \frac{1}{2\sqrt{\nu}}\)d.
$$

The parameter $\nu = \nicefrac{1}{4}$ is approximately the maximizer for
$$
\frac{32d}{(1+\omega)\E[\hat{b}]} = \frac{32}{(1+\nu)\, \(\log 3 + \frac{1}{2\sqrt{\nu}}\)} \approx 9.9,
$$
which shows the gain in total communication complexity. In other words, the scheme communicates $\(1+\log 3\)d\approx 2.6d$ bits in each iteration (about $12$ times less than without compression), but needs $1+\omega=\nicefrac{5}{4}$ times more iterations.

\begin{table*}[t]
\caption{Total communication savings due to unbiased compression method.}
\label{table:comp-methods}
\vskip 0.15in
\begin{center}
\begin{small}
\begin{sc}
\begin{tabular}{lccccc}
\toprule
Compression Method & Bits $\E[b]$ & $\times(1+\omega)$ & $\beta\eqdef \E[b]/32d$ & savings $\times(1+\omega)\beta$ \\
\midrule
No compression (base)   & $32d$ & 1 & 1 & 1 \\          
Random sparsification   & $32k + \log_2\binom{d}{k}$ & $\nicefrac{d}{k}$ & $>\nicefrac{k}{d}$ & $>1$ \\
Ternary Quantization    & $\approx d\log_2 3$ & $\sqrt{d}$ & $\nicefrac{1}{20.2}\;(0.05)$ & $\sqrt{d}/20.2$ \\
Standard Dithering      & $\approx 2.8d$ & 2 & $\nicefrac{1}{11.4}\;(0.087)$ & $\nicefrac{1}{5.7}\;(0.175)$ \\
Natural Compression     & $9d$ & {\color{PineGreen} $\nicefrac{9}{8}\;(1.125)$} & $\nicefrac{1}{3.5}\;(0.28)$ & $\nicefrac{1}{3.1}\;(0.31)$ \\
\textbf{Randomized SD (new)} & {\color{PineGreen} $\approx 2.6d$} & $\nicefrac{5}{4}\;(1.25)$ & {\color{PineGreen} $\nicefrac{1}{12.3}\;(0.081)$} & {\color{PineGreen} $\nicefrac{1}{9.9}\;(0.10)$} \\
\bottomrule
\end{tabular}
\end{sc}
\end{small}
\end{center}
\label{table:total-savings}
\vskip -0.1in
\end{table*}

\end{document}